\documentclass[twoside]{article}

\usepackage[accepted]{aistats_cr_custom}
\usepackage[round]{natbib}
\renewcommand{\bibname}{References}
\renewcommand{\bibsection}{\subsubsection*{\bibname}}

\usepackage{amsmath,amsfonts,bm}

\def\eqref#1{equation~\ref{#1}}

\def\1{\bm{1}}

\DeclareMathAlphabet{\mathsfit}{\encodingdefault}{\sfdefault}{m}{sl}
\SetMathAlphabet{\mathsfit}{bold}{\encodingdefault}{\sfdefault}{bx}{n}

\newcommand{\E}{\mathbb{E}}

\newcommand{\R}{\mathbb{R}}

\DeclareMathOperator{\Tr}{Tr}

\newcommand{\beq}{\begin{equation}}
\newcommand{\eeq}{\end{equation}}
\newcommand{\be}{\begin{equation}}
\newcommand{\ee}{\end{equation}}
\newcommand{\beqa}{\begin{eqnarray}}
\newcommand{\eeqa}{\end{eqnarray}}
\newcommand{\bean}{\begin{eqnarray*}}
\newcommand{\eean}{\end{eqnarray*}}

\newcommand{\cC}{{\mathcal C}}

\newcommand{\cF}{{\mathcal F}}
\newcommand{\cG}{{\mathcal G}}

\newcommand{\cN}{{\mathcal N}}

\newcommand{\cR}{{\mathcal R}}
\newcommand{\cS}{{\mathcal S}}
\newcommand{\cT}{{\mathcal T}}

\newcommand{\cW}{{\mathcal W}}
\newcommand{\cX}{{\mathcal X}}
\newcommand{\cY}{{\mathcal Y}}
\newcommand{\cZ}{{\mathcal Z}}

\renewcommand{\R}{\mathbb{R}}

\renewcommand{\E}{\mathbb{E}}

\def\w{\wedge}

\def\f{\frac}

\def\maps{\colon}

\def\to{\rightarrow}

\def\1{\mathds{1}}

\def\x{\mathbf{x}}

\def\w{\mathbf{w}}
\def\f{\mathbf{f}}
\def\z{\mathbf{z}}
\def\y{\bm y}

\def\bzeta{\bm\zeta}
\def\bxi{\bm\xi}
\def\x{\mathbf{x}}

\def\K{\bm K}
\def\u{\bm u}
\def\v{\bm v}

\def\Tr{\mathrm{Tr}}

\def\act{{\, \triangleright\, }}

\usepackage{amsmath}

\usepackage{bm,color}
\definecolor{mydarkblue}{rgb}{0,0.08,0.45}
\definecolor{layin}{rgb}{0.267004, 0.004874, 0.329415}
\definecolor{layinter}{rgb}{0.127568, 0.566949, 0.550556}
\definecolor{layout}{rgb}{0.993248, 0.906157, 0.143936}

\usepackage[utf8]{inputenc} 
\usepackage[T1]{fontenc}    
\usepackage[colorlinks=true,
    linkcolor=mydarkblue,
    citecolor=mydarkblue,
    filecolor=mydarkblue,
    urlcolor=mydarkblue]{hyperref}
\usepackage{url}

\usepackage{enumitem}
\usepackage{booktabs}       
\usepackage{amsfonts}       
\usepackage{nicefrac}       
\usepackage{microtype}      
\usepackage{amsthm}
\usepackage{amsmath}
\usepackage{wrapfig}
\usepackage{caption}
\usepackage{subcaption}
\usepackage{algorithm}
\usepackage[noend]{algpseudocode}
\usepackage{graphicx}
\usepackage{grffile}
\usepackage{float}
\usepackage{mathtools}

\newcommand{\cut}[1]{}

\renewcommand{\u}{\bm u}
\renewcommand{\v}{\bm v}

\usepackage{chngcntr}
\usepackage{apptools}
\newtheorem{theo}{Theorem}
\newtheorem{lem}[theo]{Lemma}
\newtheorem{prop}[theo]{Proposition}

\begin{document}
%

%
\runningauthor{Baratin, George, Laurent, Hjelm, Lajoie, Vincent, Lacoste-Julien}

\twocolumn[

\aistatstitle{Implicit Regularization via Neural Feature Alignment}

\aistatsauthor{Aristide Baratin$^{1 \ast}$ \And Thomas George$^{1\ast}$\And  C\'esar Laurent$^1$ \And R Devon Hjelm$^{2,1}$}

\vspace{0.2cm} 

\aistatsauthor{Guillaume Lajoie$^1$ \And Pascal Vincent$^{1, 3}$  \And Simon Lacoste-Julien$^{1, 3}$}

\vspace{0.2cm} 

\aistatsaddress{$^1$ Mila, Universit\'e de Montr\'eal \quad $^2$ Microsoft Research \quad $^3$ Canada CIFAR AI chair}

]

\begin{abstract}
We approach the problem of implicit regularization in deep learning from a geometrical viewpoint. 
We highlight a regularization effect induced by a dynamical alignment of the neural tangent features introduced by \citet{NTK}, along a small number of task-relevant directions. This can be interpreted as a combined mechanism of  feature selection and compression. By extrapolating a new analysis of Rademacher complexity bounds for linear models,  we motivate and study a heuristic complexity measure that captures this phenomenon, in terms of sequences of tangent kernel classes along optimization paths. The code for our experiments is available as \href{https://github.com/tfjgeorge/ntk_alignment}{https://github.com/tfjgeorge/ntk\_alignment}.
\end{abstract}

\section{Introduction}

One important property of deep neural networks is their ability to generalize well on real data.
Surprisingly, this is even true with very high-capacity networks \emph{without explicit regularization}~\citep{neyshabur2014search, understanding_DL, Hoffer2017}.
This seems at odds with the usual understanding of the bias-variance trade-off \citep{geman, Brady_variance, 
Belkin_variance}:  highly complex models are expected to overfit the training data and perform poorly on test data~\citep{hastie_09}. 
Solving this apparent paradox requires understanding the various learning biases induced by the training procedure, which can act as implicit regularizers~\citep{neyshabur2014search, neyshabur2017norm}.

In this paper, we help clarify one such implicit regularization mechanism, by examining the evolution of the {\it neural tangent features} \citep{NTK} learned by the network along the optimization paths.
Our results can be understood  from two complementary perspectives:  a {\it geometric} perspective --  the (uncentered) covariance of the tangent features defines a metric on the function class, akin to the Fisher information metric \citep[e.g.,][]{amari2016}; and a {\it functional} perspective -- through the tangent kernel and its RKHS. In standard supervised classification settings, our main observation is a dynamical alignment of the tangent features along a small number of task-relevant directions during training.  We interpret this phenomenon as a combined mechanism of {\it feature selection} and {\it compression}. The intuition motivating this work is that such a mechanism  allows large models to adapt their capacity to the task, which in turn underpins their generalization abilities.

Specifically, our main contributions are as  follows:
\begin{enumerate}[leftmargin=*] 
 \item Through experiments with various architectures on MNIST and CIFAR10,  we give empirical insights on how the tangent features and their kernel adapt to the task during training (Section  \ref{sec:nonlin}). We observe in particular a sharp increase of the  anisotropy of their spectrum  early in training, as well as an  increasing similarity with the class labels, as measured by {\it centered kernel alignment}   \citep{
 Cortes:2012}.
 \item Drawing upon intuitions from linear models (Section \ref{sec:lin}), we argue that such a dynamical alignment acts as \emph{implicit regularizer}. We motivate a new heuristic complexity measure which captures this phenomenon, and empirically show better correlation with generalization compared to various measures proposed in the recent literature (Section \ref{sec:complexity_measure}).
 \end{enumerate}

\section{Preliminaries}
\label{sec:prelim}

\begin{figure*}
 \centering
\begin{subfigure}[t]{0.2\textwidth}
\centering
\includegraphics[width=0.99\textwidth]{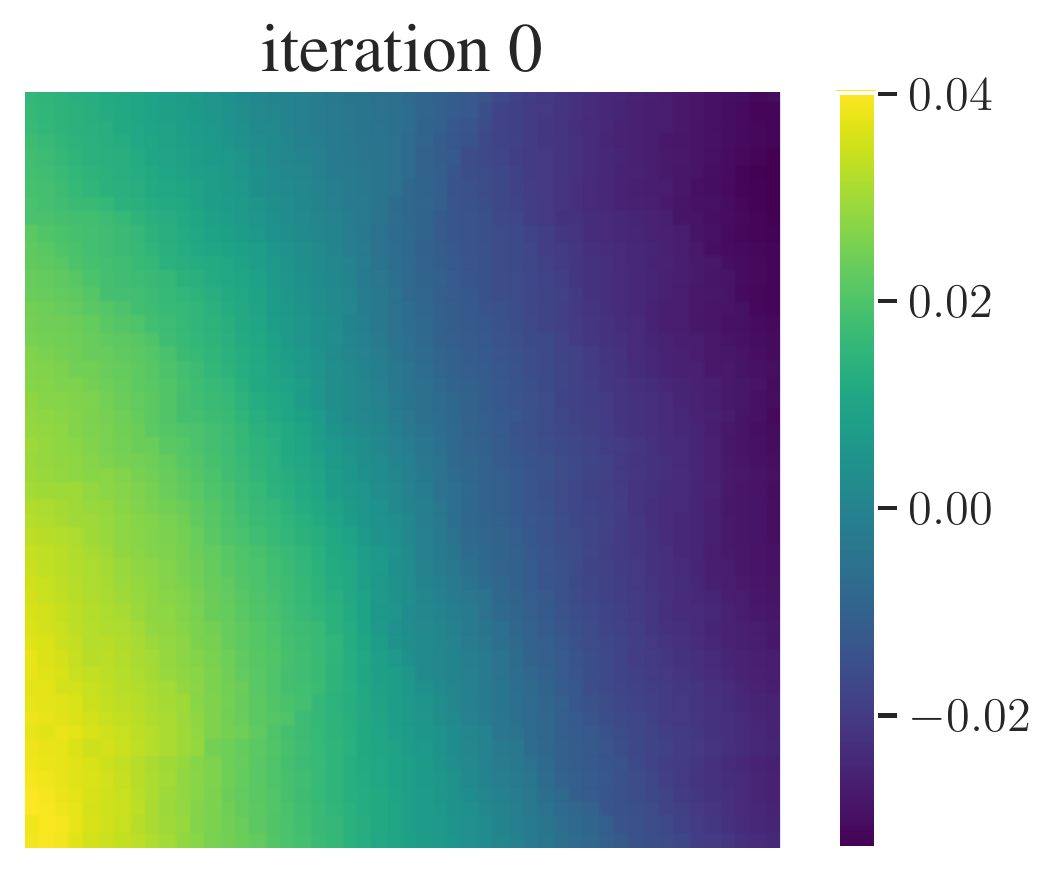}
\end{subfigure}
\begin{subfigure}[t]{0.2\textwidth}
\centering
\includegraphics[width=0.99\textwidth]{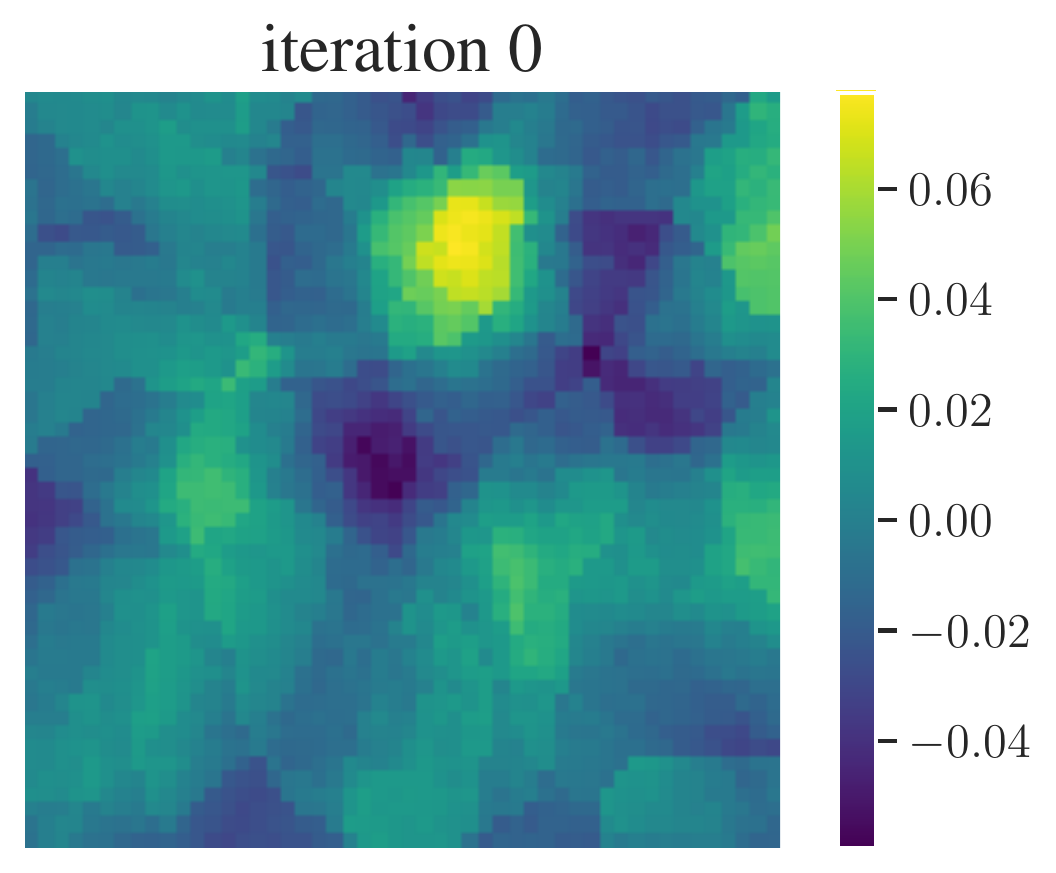}
\end{subfigure}
\begin{subfigure}[t]{0.2\textwidth}
\centering
\includegraphics[width=0.99\textwidth]{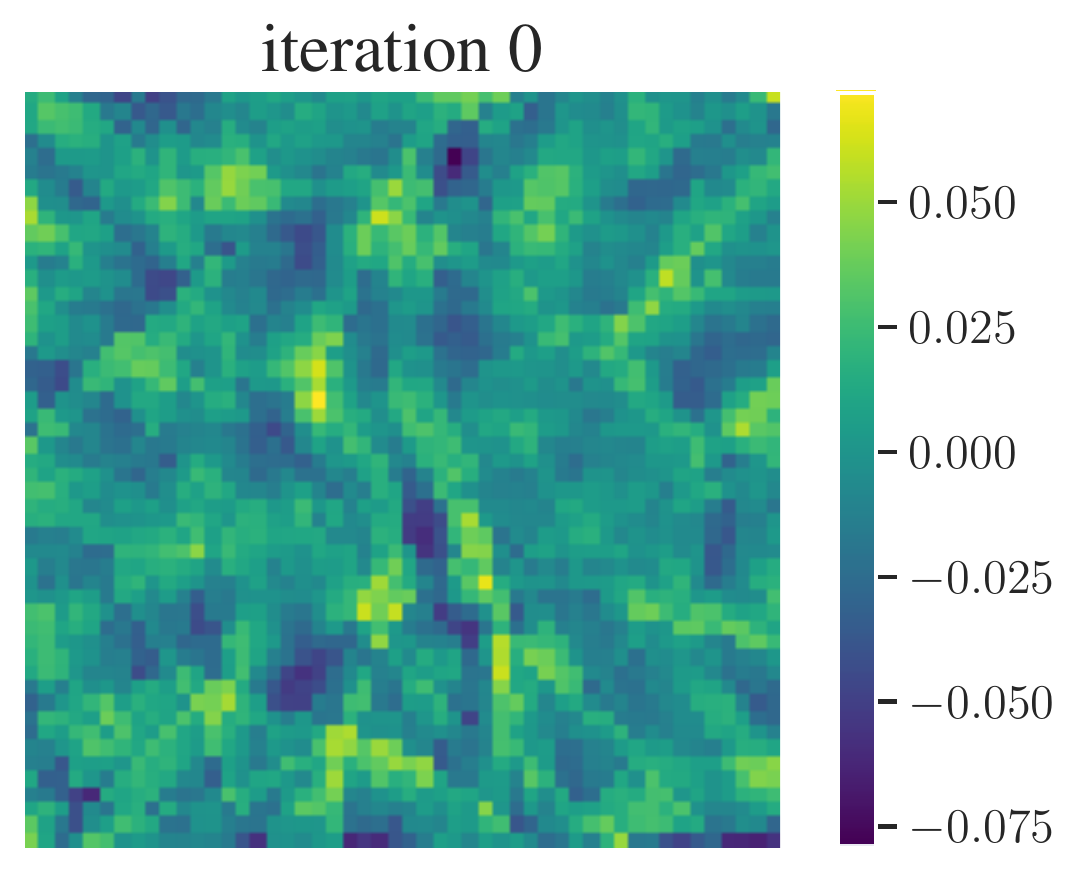}
\end{subfigure}
\begin{subfigure}[t]{0.2\textwidth}
\centering
\includegraphics[width=0.99\textwidth]{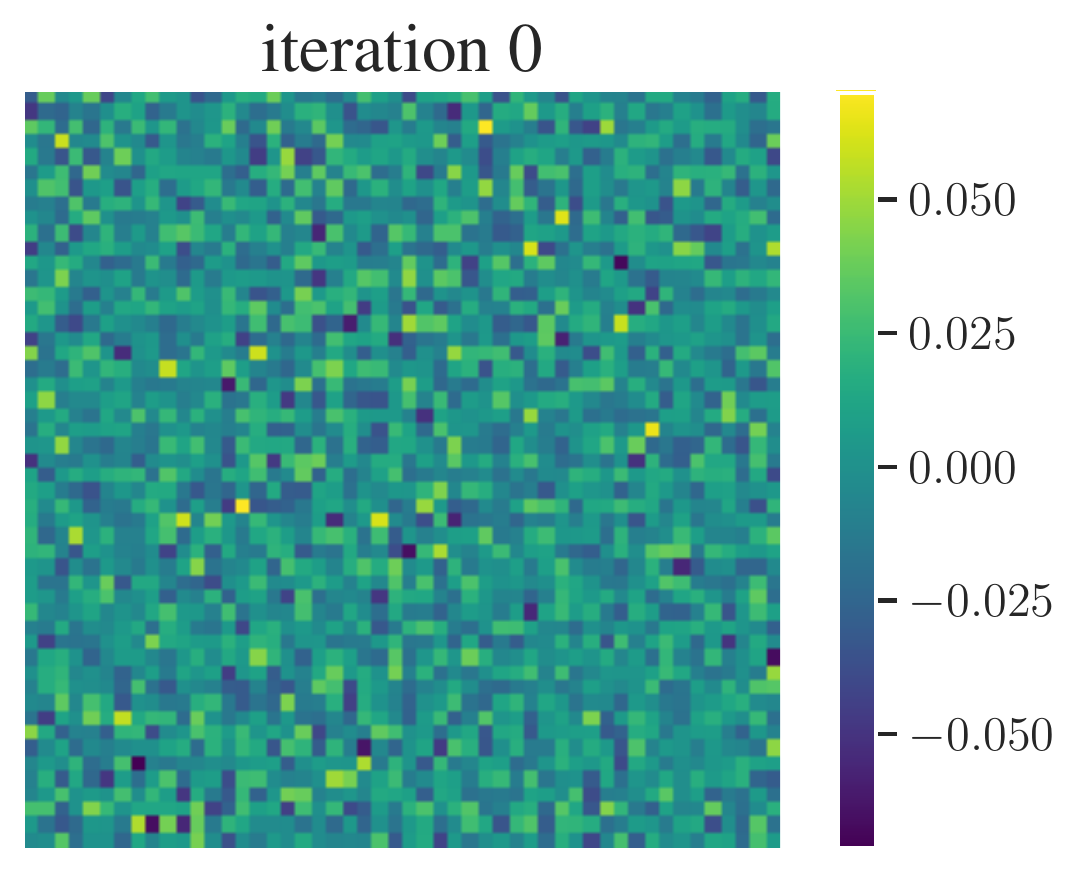}
\end{subfigure}
 \centering
\begin{subfigure}[t]{0.2\textwidth}
\centering
\includegraphics[width=0.99\textwidth]{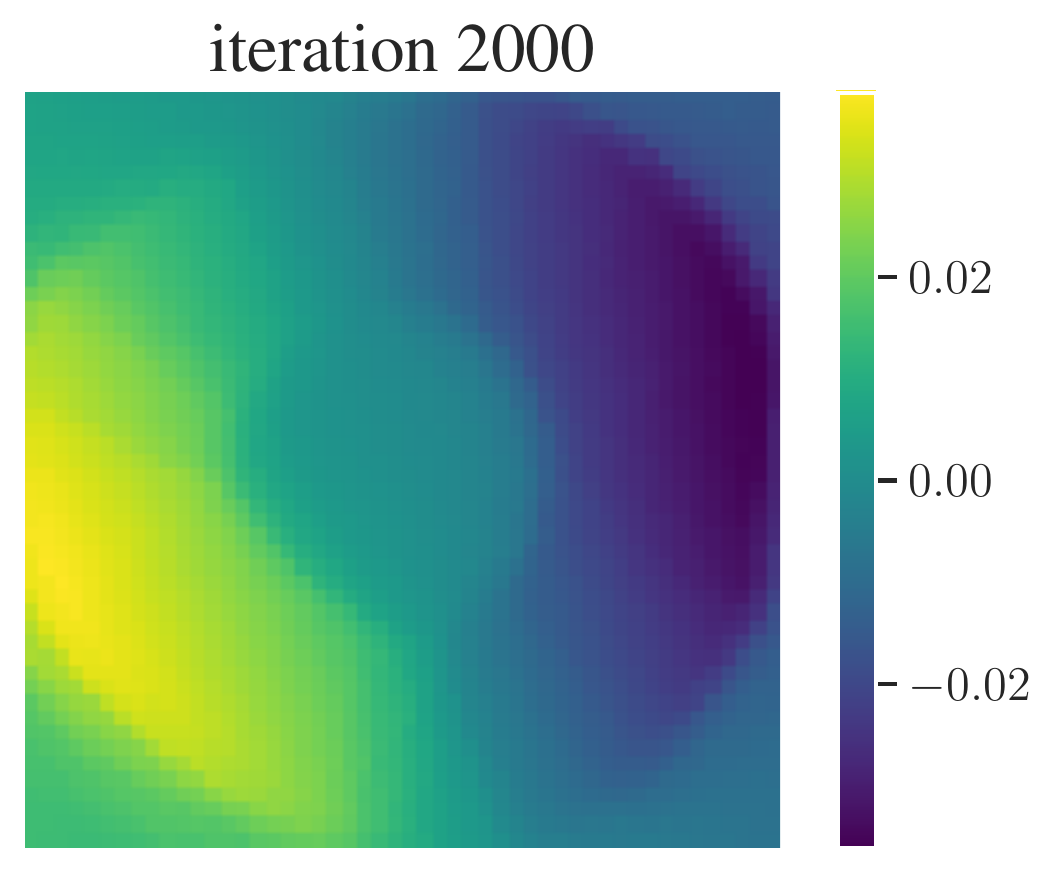}
\end{subfigure}
\begin{subfigure}[t]{0.2\textwidth}
\centering
\includegraphics[width=0.99\textwidth]{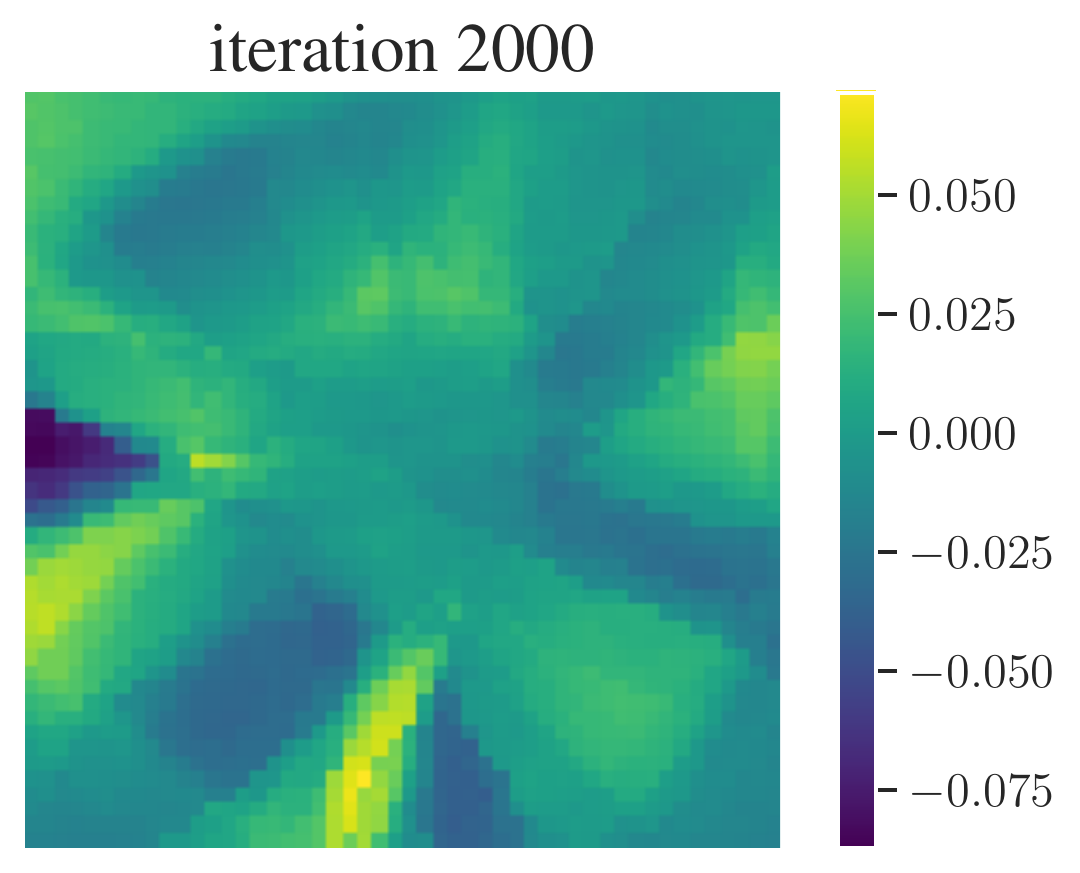}
\end{subfigure}
\begin{subfigure}[t]{0.2\textwidth}
\centering
\includegraphics[width=0.99\textwidth]{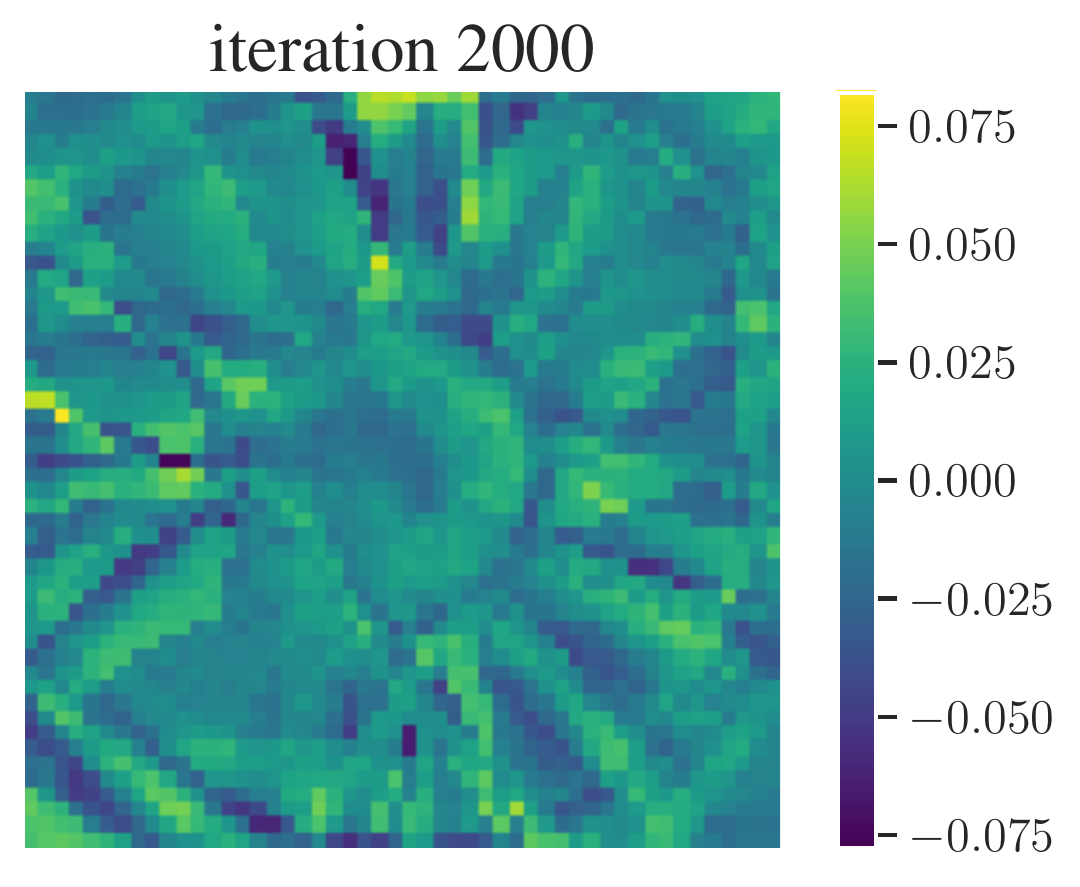}
\end{subfigure}
\begin{subfigure}[t]{0.2\textwidth}
\centering
\includegraphics[width=0.99\textwidth]{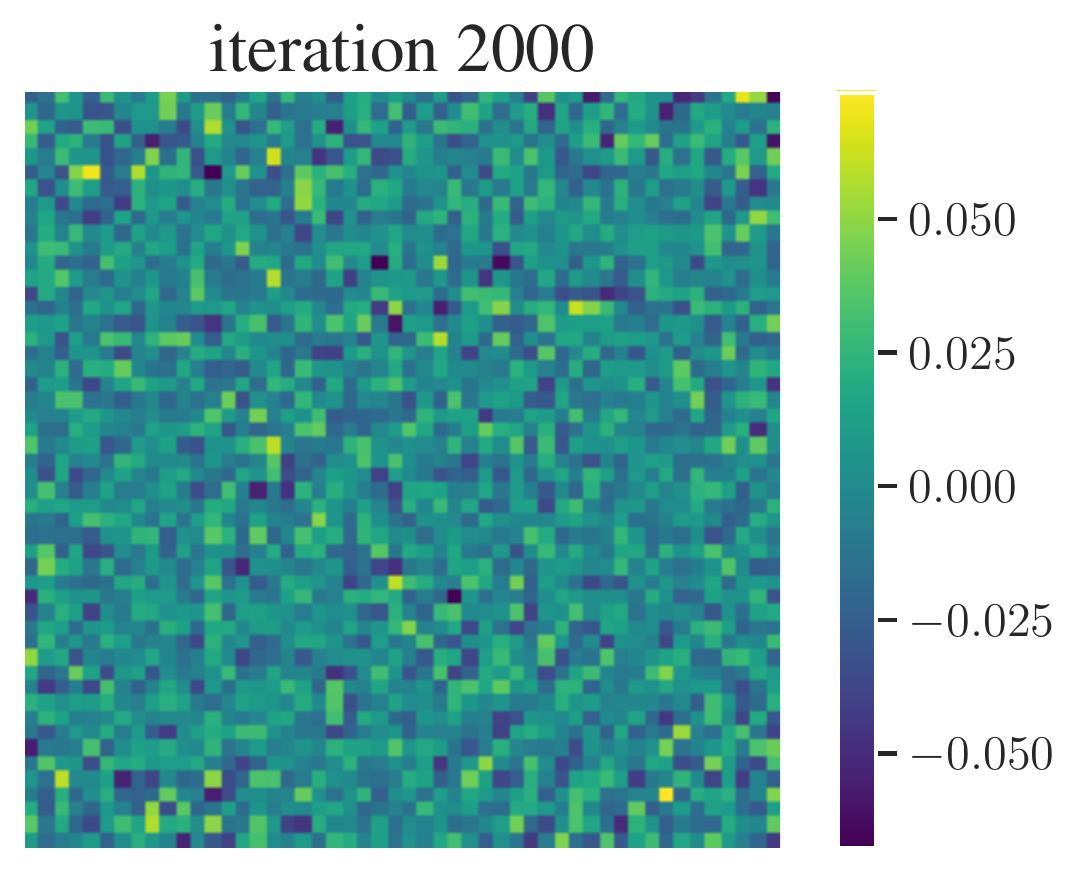}
\end{subfigure}
 \centering
\begin{subfigure}[t]{0.2\textwidth}
\centering
\includegraphics[width=0.99\textwidth]{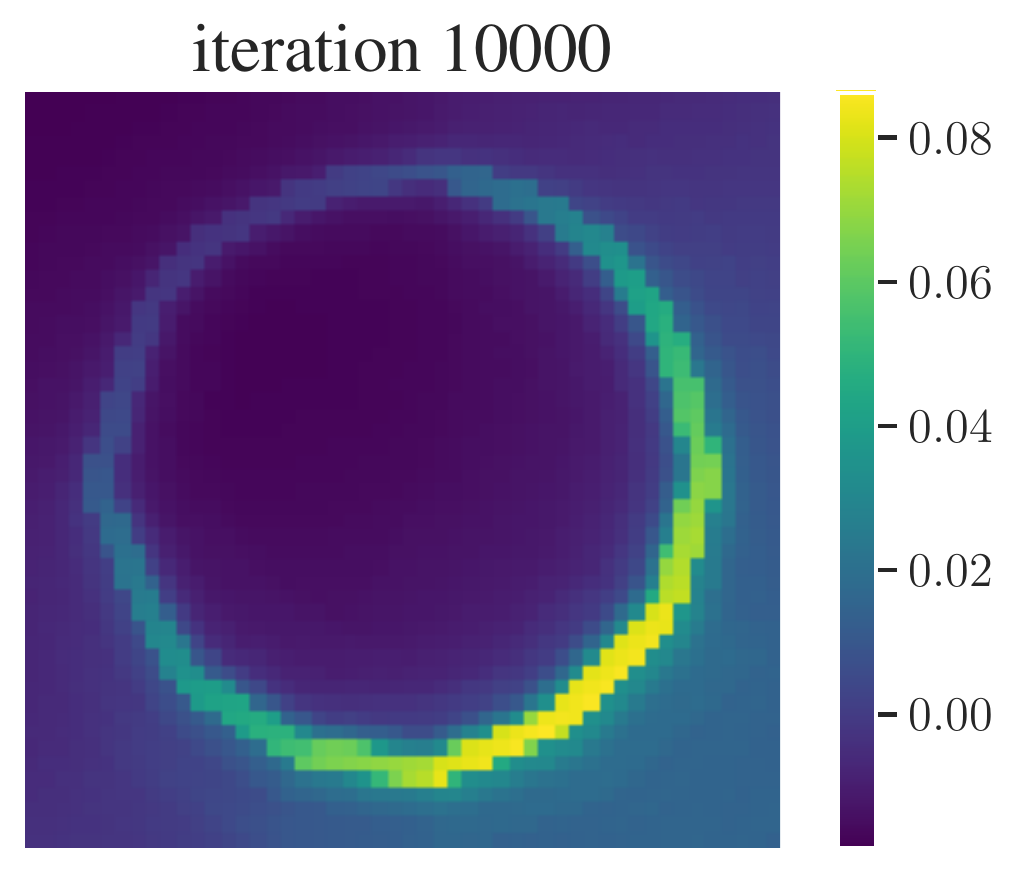}
\caption*{\small \hspace{-0.5cm} Component 0}
\end{subfigure}
\begin{subfigure}[t]{0.2\textwidth}
\centering
\includegraphics[width=0.99\textwidth]{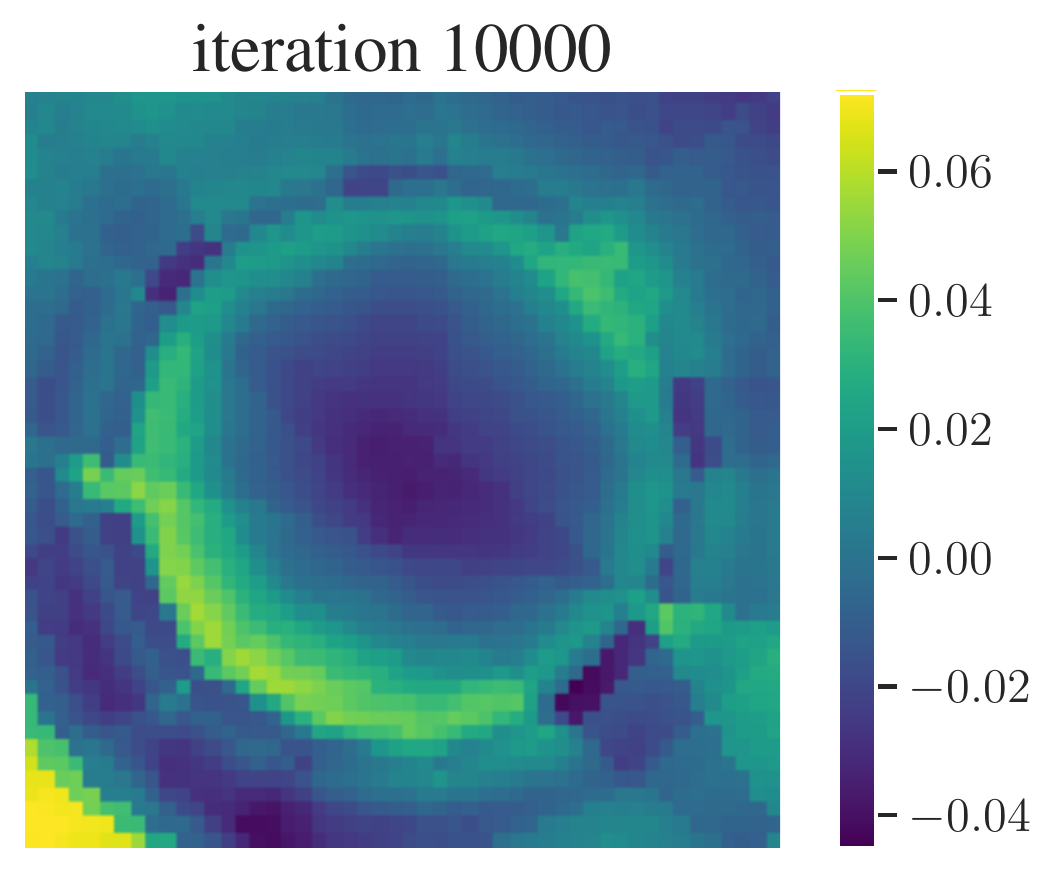}
\caption*{\small \hspace{-0.5cm} Component 20}
\end{subfigure}
\begin{subfigure}[t]{0.2\textwidth}
\centering
\includegraphics[width=0.99\textwidth]{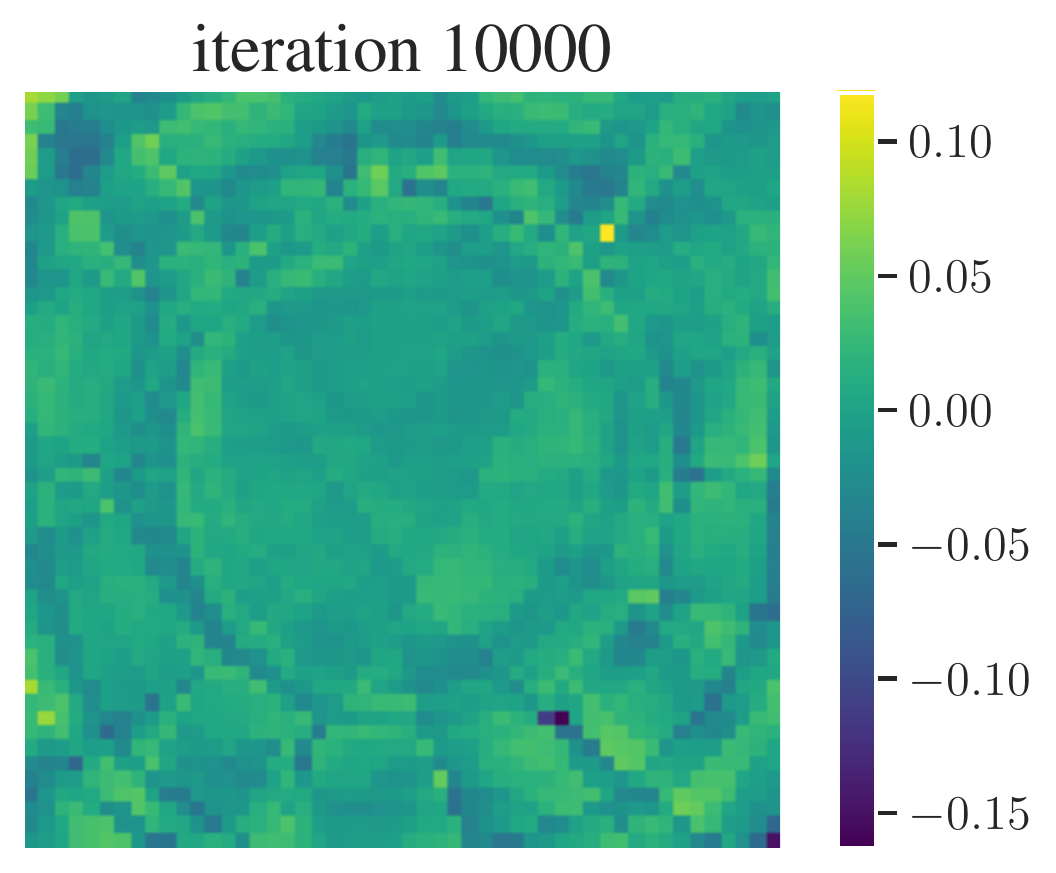}
\caption*{\small \hspace{-0.5cm} Component 100}
\end{subfigure}
\begin{subfigure}[t]{0.2\textwidth}
\centering
\includegraphics[width=0.99\textwidth]{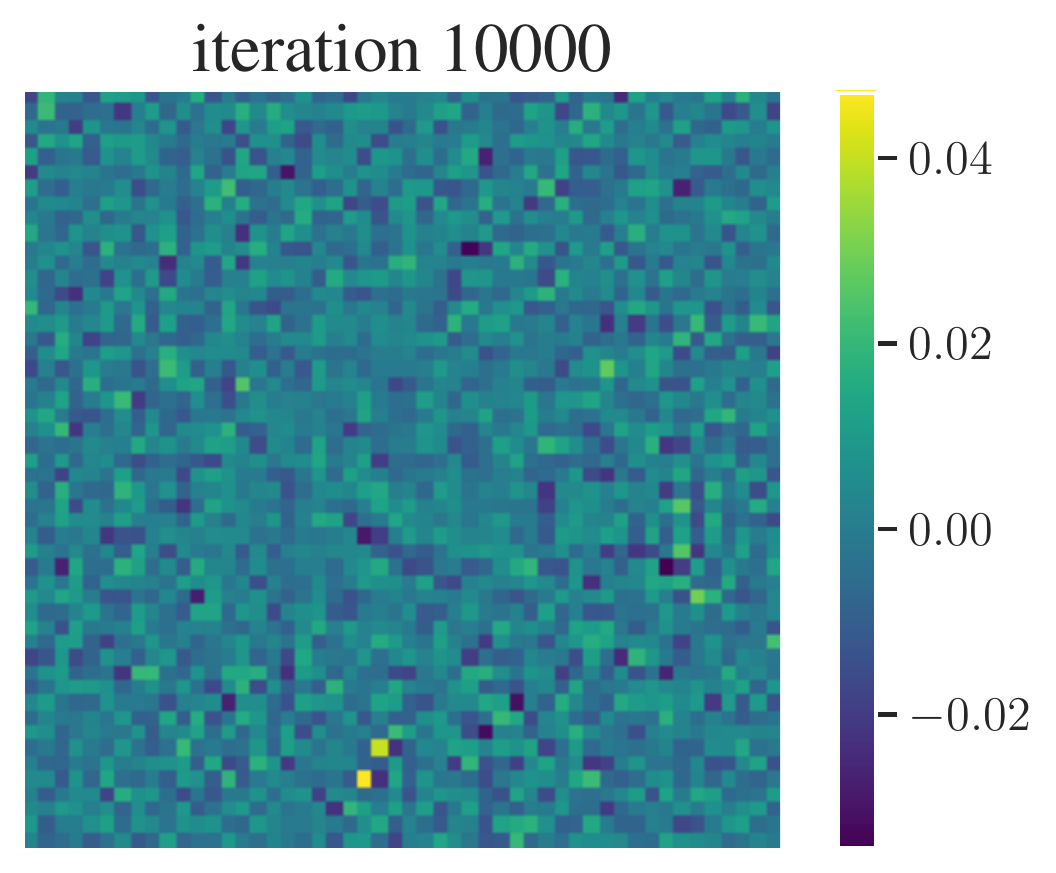}
\caption*{\small \hspace{-0.5cm} Component 1000}
\end{subfigure}

\caption{\small Evolution of eigenfunctions of the tangent kernel, ranked in nonincreasing order of the eigenvalues ({\bf in columns}), at various iterations during training ({\bf in rows}), for the $2d$ Disk dataset. After a number of iterations, we observe modes corresponding to the class structure (e.g. boundary circle) in the top eigenfunctions. Combined with an increasing anistropy of the spectrum (e.g $\lambda_{20} / \lambda_{1} = 1.5\%$ at iteration $0$, $0.2\%$ at iteration $2000$), this illustrates a  stretch of the tangent kernel, hence a (soft) compression of the model, along a small number of  features that are highly correlated with the classes.} 
\label{fig:NTK_Disk}
\end{figure*}

Let $\cF$ be a class of functions (e.g a neural network) parametrized by $\w \in \R^P$. We restrict here to {\it scalar} functions $f_\w \maps \cX \to \R$ to keep notation light.\footnote{The extension to vector-valued functions, relevant for the multiclass classification setting, is  presented in Appendix \ref{appendix:geom}, along with more  mathematical details.}

{\bf Tangent Features.} 
We define the {\bf tangent features} 
as the function gradients
w.r.t the parameters,
\beq  \label{eq:tang_feat}
\Phi_\w(\x):=\nabla_{\!\w} f_\w(\x) \in \R^P. 
\eeq 
The corresponding kernel $k_\w(\x,  \tilde{\x}) = \langle \Phi_\w(\x), \Phi_\w(\tilde{\x}) \rangle$ is the  {\bf tangent kernel} \citep{NTK}.  Intuitively, the tangent features govern how small changes in parameter affect the function's outputs, 
\beq \label{eq:fct_updates}
\delta f_{\w}(\x) =  \langle\delta \w, \Phi_{\w}(\x)\rangle + O(\|\delta \w\|^2).\eeq 
More formally, the (uncentered) covariance matrix $g_\w = \E_{\x\sim\rho}\left[ \Phi_{\!\w}(\x) \Phi_{\!\w}(\x)^{\!\top}\right]$ w.r.t the input distribution $\rho$ 
acts as a {\bf metric tensor} on $\cF$: assuming $\cF \subset L^2(\rho)$, this is the metric induced on $\cF$ by pullback of the $L^2$ scalar product. 
It characterizes the geometry of the function class $\cF$.
Metric (as  symmetric $P \times P$  matrices) and tangent kernels  (as rank $P$ integral operators) share the same spectrum (see Prop~\ref{sec_app:gkspec} in Appendix~\ref{appsec:spectral}).

{\bf Spectral Bias.}
\label{sec:spectral}
The structure of the tangent features impacts the evolution of the function during training. To formalize this, we introduce the covariance eigenvalue decomposition $g_\w \!=\! \sum_{j=1}^{P} \lambda_{\w j} \v_{\w j} \v_{\w j}^{\!\top}$, which summarizes the predominant directions in parameter space.
Given $n$ input samples $(\x_i)$ and $\f_\w \!\in\! \R^n$ the vector of outputs $f_\w(\x_i)$, consider gradient descent updates $\delta \w_{\!\mbox{\tiny GD}} \!=\! - \eta \nabla_{\!\w} L$ for some cost function $L \!:=\! L(\f_\w)$. The following elementary result (see Appendix \ref{appendix:spec_bias})  shows how the  corresponding function updates in the linear approximation (\ref{eq:fct_updates}), $\delta f_{{\!\mbox{\tiny GD}}}(\x) := \langle \delta {\w_{\!\mbox{\tiny GD}}}, \Phi_\w(\x)\rangle$,
decompose in the {\bf eigenbasis}\footnote{The functions $(u_{\w j})_{j=1}^P$ 
form an orthonormal family in $L^2(\rho)$, i.e. $\E_{\x\sim\rho}[u_{\w j} u_{\w j'}] = \delta_{j j'}$, 
and yield the spectral decomposition $k_\w(\x, \tilde{\x}) = \sum_{j=1}^{P} \lambda_{\w j} u_{\w j}(\x) u_{\w j}(\tilde{\x})$ of the tangent kernel as an integral operator (see Appendix~\ref{appsec:spectral}). 
}
of the tangent kernel:
\beq  \label{eq:pca_comp}
u_{\w j}(\x) = \frac{1}{\sqrt{\lambda_{\w j}}} \langle \v_{\w j}, \Phi_\w(\x)\rangle
\eeq 

\begin{lem}[Local Spectral Bias]
\label{lemma:PCAlinear} 
The function updates  decompose as $\delta f_{\mbox{\tiny GD}}(\x) = \sum_{j=1}^P  \delta f_j u_{\w j}(\x)$ with
\beq \label{eq:gd_spectral}
\delta f_j = -\eta \lambda_{\w j} (\u_{\w j}^{\!\top} \nabla_{\!\f_\w} L),
\eeq 
where $\u_{\w j}=[u_{\w j}(\x_1), \cdots u_{\w j}(\x_n)]^\top \in \R^n$ and $\nabla_{\!\f_\w}$ denotes the gradient w.r.t the sample outputs. 
\end{lem}
This illustrates how, from the point of view of function space, the metric/tangent kernel eigenvalues act as a mode-specific rescaling $\eta \lambda_{\w j}$ of the learning rate.\footnote{Intuitively, the eigenvalue $\lambda_{\w j}$ can be thought of as defining a local `learning speed' for the mode $j$.} 
This is a local version of a well-known bias for linear models trained by gradient descent (e.g in linear regression, see Appendix \ref{appendix:spec_bias_linear}), which prioritizes learning functions within the top eigenspaces of the kernel.  
Several recent works \citep{NIPS2019_9449, Basri2019, Yang2019} investigated such bias for neural networks, in {\it linearized} regimes where the tangent kernel remains constant during training \citep{NTK, du2018gradient, Allen-ZhuNTK}. 
As a simple example, for a randomly initialized  MLP on  1D uniform data, Fig.~\ref{fig:NTK_freq} in Appendix \ref{appendix:spec_bias} shows an alignment of the  tangent kernel eigenfunctions with Fourier modes of increasing frequency, in line with prior empirical observations \citep{SpectralBias, Fprinciple} of a `spectral bias' towards low-frequency functions.

\begin{figure*}[t]
\centering
\includegraphics[width=.33\linewidth]{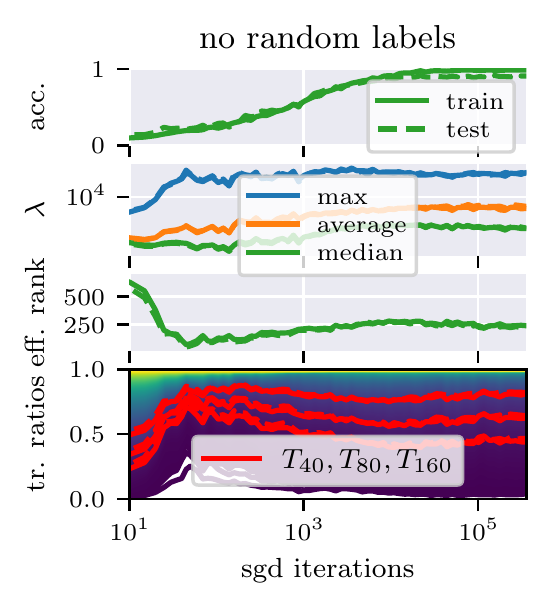}\includegraphics[width=.33\linewidth]{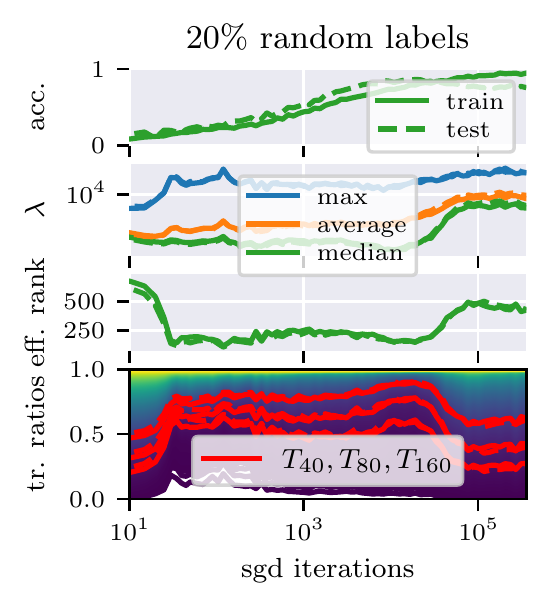}\includegraphics[width=.33\linewidth]{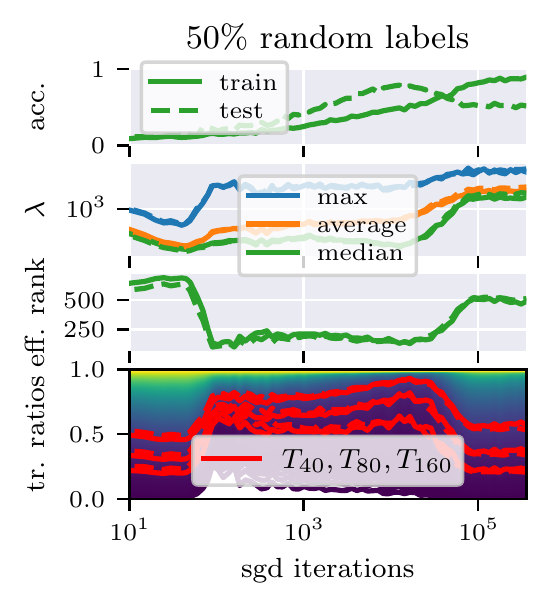}
\caption{\small Evolution of the tangent kernel {\bf spectrum} (max, average and median eigenvalues),  {\bf effective rank} (\ref{eq:effrank}) and {\bf trace ratios} (\ref{eq:tr_ratios}) during training of a VGG19 on CIFAR10 with various ratio of random labels, using cross-entropy and SGD with batch size $100$, learning rate $0.01$ and momentum $0.9$. Tangent kernels are evaluated on batches of size 100 from both the training set (solid lines) and the test set (dashed lines). The plots in the top row show train/test accuracy.   
}
\label{fig:funcclassTKspectrum}
\end{figure*}
\begin{figure*}
\begin{subfigure}[t]{0.33\linewidth}
\centering
\includegraphics[width=\linewidth]{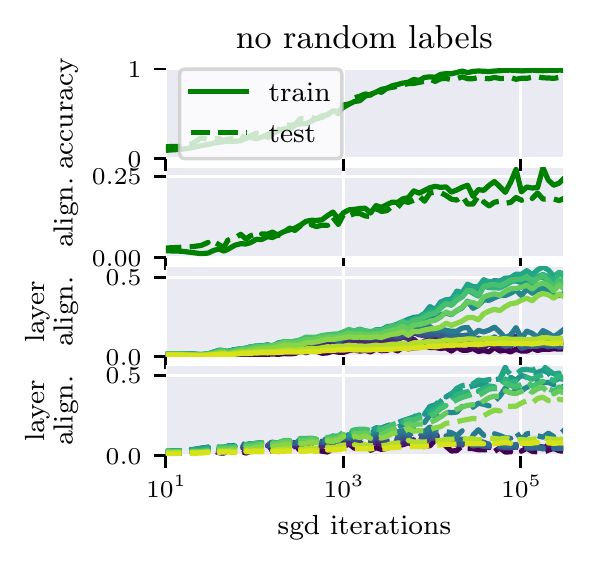}
\end{subfigure}
\begin{subfigure}[t]{0.33\linewidth}
\centering
\includegraphics[width=\linewidth]{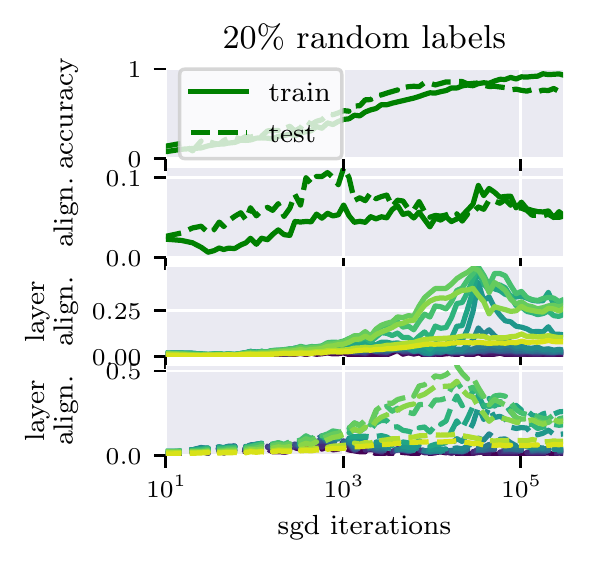}
\end{subfigure}
\begin{subfigure}[t]{0.33\linewidth}
\centering
\includegraphics[width=\linewidth]{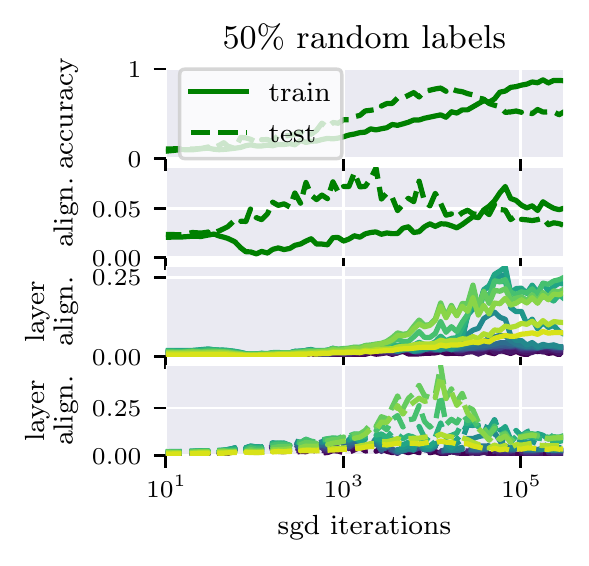}
\end{subfigure}
\caption{\small 
Evolution of the (tangent)  {\bf feature alignment with class labels} as measured by CKA  (\ref{eq:CKA}), during training of a VGG19 on CIFAR10 (same setup as in Fig.~\ref{fig:funcclassTKspectrum}). 
Tangent kernels and label vectors are evaluated on batches of size $100$ from both the training set (solid lines) and the test set (dashed lines). The plots in the last two rows  show the alignment of tangent features associated to {\it each layer}.  
Layers are mapped to colors sequentially from input layer (\textcolor{layin}{\textbf{-}}), through intermediate layers (\textcolor{layinter}{\textbf{-}}), to output layer (\textcolor{layout}{\textbf{-}}).
See Fig.~\ref{fig:cka_alignment_additional} and \ref{fig:align_varying_depth} in Appendix~\ref{appendix:exps} for additional architectures and datasets.} 
\label{fig:kernel_align}
\end{figure*}

\paragraph{Tangent Features Adapt to the Task.} By contrast, our aim in this paper is to highlight and discuss {\it non-linear} effects, in the (standard) regime where the tangent features and their kernel evolve during training \citep[e.g.,][]{Geiger2019, Kernel_rich}. 

As a first illustration of such effects, Fig.~\ref{fig:NTK_Disk} shows visualizations of eigenfunctions of the tangent kernel (ranked in nonincreasing order of the eigenvalues), during training of a 6-layer deep 256-unit wide MLP by gradient descent of the binary cross entropy loss,  on a simple classification task: 
$y(\x) = \pm 1$ depending on whether $\x \sim \mbox{Unif}[-1,1]^2$ is in  the centered disk of radius $\sqrt{2/\pi}$ (details in Appendix \ref{appendix:disk_principal_components}). After a number of iterations, we observe (rotation invariant) modes corresponding to the class structure (e.g. boundary circle) showing up in the {\it top} eigenfunctions of the learned kernel. We also note an increasing spectrum anisotropy -- for example, the ratio $\lambda_{20} / \lambda_{1}$, which is $1.5\%$ at iteration $0$, has dropped to $0.2\%$ at iteration $2000$. The interpretation is that the tangent kernel (and the metric) {\it stretch} along a relatively small number of directions that are  highly correlated with the classes during training.
We quantify and investigate this effect in more detail
below.

\section{Neural Feature Alignment}  
\label{sec:nonlin}

In this section, we study in more detail the evolution of the tangent features during training. Our main results are to highlight $(i)$ a sharp increase of the anisotropy of their spectrum early in training; $(ii)$ an increasing similarity with the class labels, as measured by {\bf centered kernel alignment} (CKA) \citep{Cristianini2002, Cortes:2012}.   We interpret this as a combined mechanism of feature selection and model compression.

\subsection{Setup} 

We run experiments on MNIST \citep{lecun2010mnist}  and CIFAR10 \citep{krizhevsky2009learning} with standard MLPs, VGG \citep{simonyan2014very} and Resnet \citep{he2016deep} architectures, trained by stochastic gradient descent (SGD) with momentum, using cross-entropy loss. We use PyTorch \citep{NEURIPS2019_9015} and NNGeometry \citep{anonymous2021nngeometry} for efficient evaluation of tangent kernels. 

In multiclass settings, 
tangent kernels evaluated on $n$ samples carry additional class indices $y \in \{1\cdots c\}$ and thus are $nc \times nc$ matrices, $(\K_\w)_{ij}^{yy'}:= k_\w(\x_i, y; \x_j, y')$ (details in Appendix \ref{secapp:sampled}). In all our experiments, we evaluate tangent kernels on mini-batches of size $n=100$  from both the training set and the test set; for $c=10$ classes, this yields kernel matrices of size $1000 \times 1000$. We report results obtained from {\it centered} tangent features $\Phi_\w(\x) \to \Phi_\w(\x) - \E_{\x} \Phi_\w(\x)$, though we obtain qualitatively similar results for uncentered features (see plots in Appendix \ref{appendix:uncentered_kernel}).

\subsection{Spectrum Evolution} \label{spec:spec_evol}

We first investigate the evolution of the tangent kernel {\it spectrum} for a VGG19 on CIFAR 10, trained with and without label noise (Fig.~\ref{fig:funcclassTKspectrum}).  The take away is an anisotropic increase of the spectrum during training. We report results for kernels evaluated on training examples (solid line) and test examples (dashed line).\footnote{The striking similarity of the plots for train and test kernels suggests that the spectrum of empirical tangent kernels is robust to sampling variations in our setting.}

The first observation is a significant {\it increase} of the spectrum, early in training (note the log scale for the $x$-axis). By the time the model reaches 100$\%$ training accuracy,  the maximum and average eigenvalues (Fig.~\ref{fig:funcclassTKspectrum}, 2nd row) have gained more than 2 orders of magnitude. 

The second observation is that this evolution is highly {\it anisotropic}, i.e larger eigenvalues increase faster than lower ones. This results in a (sharp) increase of spectrum anisotropy, early in training.  We quantify this using a notion of {\bf effective rank} based on spectral entropy \citep{roy2007effective}. Given a kernel matrix $\K$ in $\R^{r \times r}$ with (strictly) positive eigenvalues $\lambda_1, \cdots, \lambda_r$, let 
$\mu_j = \lambda_j / \sum_{i=1}^r \lambda_j$ be the trace-normalized eigenvalues.   The  effective rank is defined  as $\mathrm{erank} = \exp(H({\bm \mu}))$ where $H({\bm \mu})$ is the Shannon entropy,
\beq 
\label{eq:effrank}
\mathrm{erank} = \exp(H({\bm \mu})), \,\, H({\bm \mu}) = -\sum_{j=1}^r \mu_j \log (\mu_j).
\eeq  
This effective rank is a real number between $1$ and $r$, upper bounded by $\mathrm{rank}(\K)$, which measures the `uniformity' of the spectrum through the entropy. 
We also track the  various {\bf trace ratios} 
\beq \label{eq:tr_ratios} 
T_k = \sum_{j<k} \lambda_j / \sum_j \lambda_j,
\eeq 
which quantify the relative importance of the top $k$ eigenvalues.

We note (Fig.~\ref{fig:funcclassTKspectrum}, third row) a drop of the effective rank early in training (e.g. to less than $10\%$ of its initial value in our experiments with no random labels; less than $20\%$ when half of the labels are randomized). This can also be observed from the highlighted (in red) trace ratios $T_{40}$, $T_{80}$ and $T_{160}$ (Fig.~\ref{fig:funcclassTKspectrum}, fourth row), e.g. the first top $40$ eigenvalues ($T_{40}$), over 1000 in total, accounting for more than 70$\%$ of the total trace. 

Remarkably, in the presence of high label noise,  the effective rank of the tangent kernel (and hence that of the metric) evaluated on {\it training} examples (anti)-correlates nicely with the {\it test} accuracy:  while decreasing and remaining relatively low during the learning phase (increase of test accuracy), it begins to rise again when overfitting starts (decrease of test accuracy). This suggests that this effective rank   already provides a good proxy for the  effective capacity of the network.

\subsection{Alignment to class labels} 

\begin{figure*}[t]
\centering
\begin{subfigure}[t]{0.4\linewidth}
\includegraphics[width=0.95\linewidth]{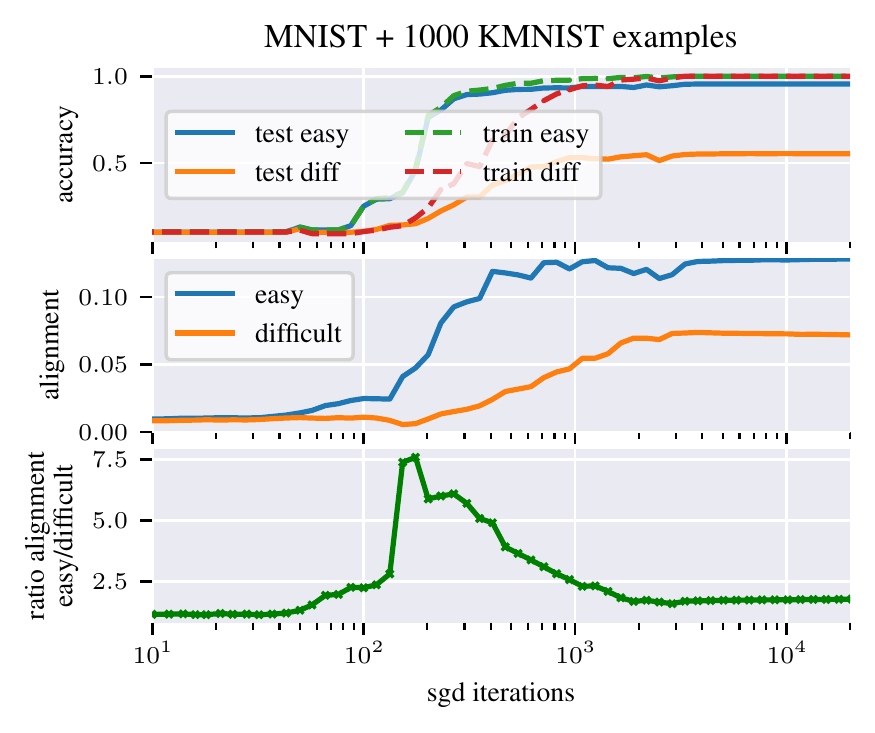}
\end{subfigure}
\begin{subfigure}[t]{0.4\linewidth}
\centering
\includegraphics[width=0.95\linewidth]{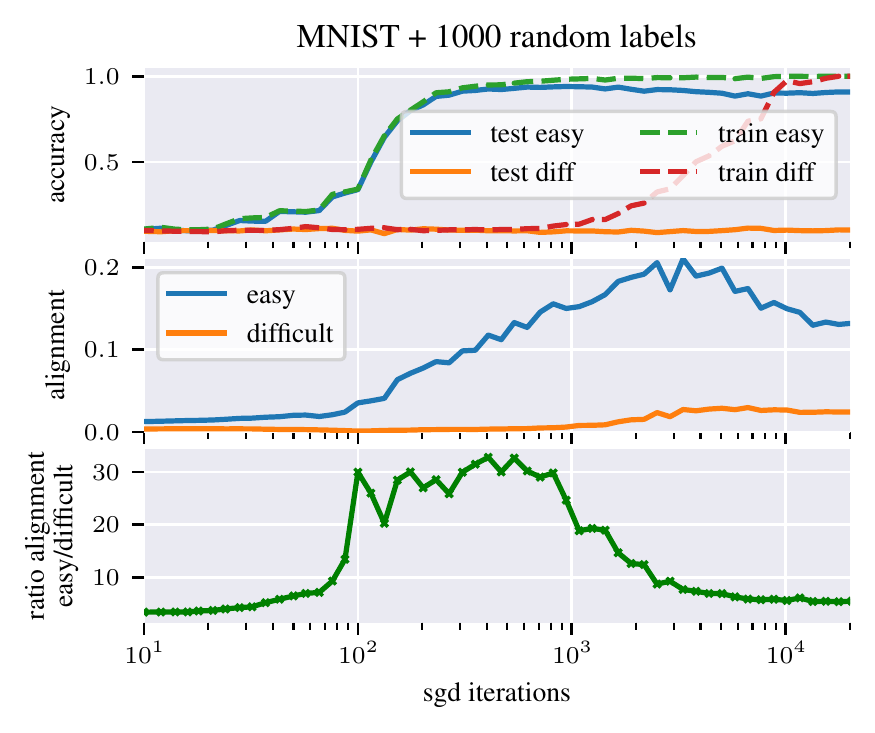}
\end{subfigure}

\caption{\small Alignment {\it easy} versus {\it difficult}: We augment a dataset composed of 10.000 \emph{easy} MNIST examples with 1000 \emph{difficult} examples from 2 different setups:  \textbf{(left)} 1000 MNIST examples with random label \textbf{(right)} 1000 KMNIST examples. We train a MLP with 6 layers of 80 hidden units using SGD with learning rate=0.02, momentum=0.9 and batch size=100. We observe that the alignment to (train) labels increases faster and to a higher value for the easy examples. 
}
\label{fig:easy-diff-alignment}
\end{figure*}

We now include the evolution of the eigenvectors in our study. We investigate the similarity of the learned tangent features with the class label through centered kernel alignment. 
Given two kernel matrices $\K$ and $\K'$ in $\R^{r\times r}$, 
it is defined as \citep{Cortes:2012}
 \beq \label{eq:CKA}
 \mathrm{CKA}(\K, \K') = \frac{\Tr[\K_c \K_c']}{\|\K_c\|_F \|\K'_c\|_F} \, \in [0, 1] 
 \eeq
where the  subscript $c$ denotes the feature centering operation, i.e. $\K_c = C \K C$ where  $C = I_r - \frac{1}{r} {\bm 1} {\bm 1}^T$ is the centering matrix, and $\| \cdot \|_F$ is the Froebenius norm.  CKA is a normalized version of the Hilbert-Schmidt Independence Criterion \citep{Gretton05measuringstatistical} designed as a dependence measure for two sets of features. The normalization makes CKA invariant under isotropic rescaling.

Let ${\bm Y} \in \R^{nc}$ be the vector resulting from the concatenation of the one-hot label representations ${\bm Y_i} \in \R^c$ of the $n$ samples.  Similarity with the labels is measured through CKA with the rank-one kernel 
$\K_{\bm Y} := {\bm Y} {\bm Y}^\top$. 
Intuitively, $\mathrm{CKA}(\K, \K_{\bm Y})$ is high when $\K$ has low (effective) rank and such that the angle between ${\bm Y}$ and its top eigenspaces is small.\footnote{In the limiting case $\mathrm{CKA}(\K, \K_{\bm Y}) = 1$, the features are all aligned with each other and parallel to ${\bm Y}$.}   Maximizing such an index has been used  as a criterion for kernel selection in the literature on learning kernels \citep{Cortes:2012}. 

With the same setup as in Section \ref{spec:spec_evol}, we observe (Fig.~\ref{fig:kernel_align}, 2nd row) an increasingly  high  CKA between the tangent kernel and the labels as training progresses. The trend is similar for other architectures and datasets (e.g., Fig.~\ref{fig:cka_alignment_additional} in Appendix~\ref{appendix:exps} shows CKA plots for MLP on MNIST and Resnets 18 on CIFAR10).  

Interestingly, in the presence of high level noise,  the CKA reaches a much higher value during the learning phase (increase of test accuracy) for tangent kernels and labels evaluated for {\it test} than for {\it train} inputs (note test labels are not randomized). Together with Equ.~\ref{eq:gd_spectral}, this suggests a stronger learning bias towards features predictive of the {\it clean} labels. This is line with empirical observations that, in the presence of noise, deep networks `learn patterns faster than noise' \citep{arpit2017closer} (see Section \ref{sec:hierarch_align} below for additional insights).  

We also report the alignments of the {\it layer-wise} tangent kernels. By construction, the tangent kernel, obtained by pairing features $\Phi_{w_p}(\x)\Phi_{w_p}(\tilde{\x})$ and summing over all parameters $w_p$ of the network, can also be expressed as the sum of layer-wise tangent kernels, $\K_\w = \sum_{\ell=1}^L \K_\w^\ell$,  where $\K_\w^\ell$ results from summing only over parameters of the layer $\ell$.
We observe a high CKA, reaching more than 0.5 for a number of intermediate layers.\footnote{We were expecting to see a gradually increasing CKA with $\ell$; we do not have any intuitive explanation for the relatively low alignment observed for the very top layers.} In the presence of high label noise, we note that CKAs tend to peak when the test accuracy does. 

\subsection{Hierarchical Alignment} 
\label{sec:hierarch_align}

A key aspect of the  generalization question concerns the articulation between learning and memorization, in the presence of noise \citep{understanding_DL} or difficult examples \citep[e.g.,][]{Liang2020}. Motivated by this,  we would like to probe the evolution of the tangent features  {\it separately} in the directions of both types of examples in such settings.  To do so, our strategy is to 
measure CKA for tangent kernels and label vectors evaluated on examples from two subsets of the same size in the training dataset -- one with `easy' examples, the other with `difficult' ones.  Our setup is to augment 10.000 MNIST training examples with 1000 difficult examples of 2 types: (\emph{i}) examples with random labels and (\emph{ii}) examples from the dataset KMNIST \citep{clanuwat2018deep}. KMNIST images present features similar to MNIST digits (grayscale handwritten characters) but 
represent Japanese characters. 

The results are shown in Fig.~\ref{fig:easy-diff-alignment}. As training progresses, we observe that the CKA on the easy examples increases faster (and to a higher value) than that on the difficult ones; in the case of the (structured) difficult examples from KMNIST, we also note an increase of the CKA later in training. This demonstrates a hierarchy in the adaptation of the kernel, measured by the ratio between both alignments. From the intuition developed in the paper (see spectral bias in Equ.(\ref{eq:gd_spectral})), we interpret this aspect of the non-linear dynamics as favoring a sequentialization of learning across patterns of different complexity (`easy patterns first'), a phenomenon analogous to one pointed out in the context  of deep linear networks \citep{Saxe14exactsolutions, Lampinen:2018wh, NIPS2019_8583}. 

\subsection{Ablation} 
\label{sec:ablation}

{\bf Effect of depth.} 
In order to study the influence of depth on alignment and test the robustness to the choice of seeds, we reproduce the experiment of the previous section for MLP with different depths, while varying parameter initialization and minibatch sampling. Our results, shown in Fig \ref{fig:align_varying_depth} (Appendix \ref{appendix:exps}), suggest that the alignment effect is magnified as depth increases. We also observe that the ratio of the maximum alignment between easy and difficult examples is increased with depth, but stays high for a smaller number of iterations. 

{\bf Effect of the learning rate.} We observed in our experiments that increasing the learning rate tend to enhance alignment effects.\footnote{Note that for wide enough networks and small enough learning rate, we expect to recover the linear  regime where the tangent features are constant during training \citep{NTK, du2018gradient, Allen-ZhuNTK}.} As an illustration, we reproduce in Fig.~\ref{appfig:funcclassTKspectrumwithtraceratios} the same plots as in Fig.~\ref{fig:funcclassTKspectrum},  for a learning rate reduced to $0.003$. We observe a similar drop of the effective rank as in Fig.~\ref{fig:funcclassTKspectrum} at the beginning of training, but to a much (about 3 times) higher value.

\section{Measuring Complexity}
\label{sec:complexity_measure}

In this section, drawing upon intuitions from linear models, we illustrate in a simple setting how the alignment of tangent features can act as implicit regularization. By extrapolating Rademacher complexity bounds  for linear models, we also motivate a new complexity measure for neural networks and compare its correlation to generalization against various measures proposed in the literature. We refer to Appendix \ref{secapp:comp_bounds} for a review of classical results, further technical details, and proofs. 

\subsection{Insights from Linear Models} 
\label{sec:lin}


\subsubsection{Setup}

We restrict here to scalar functions $f_\w(\x) \!=\! \langle \w, \Phi(\x) \rangle$ linearly parametrized by $\w \in \R^P$.
Such a function class defines a {\it constant} (tangent) kernel and geometry, as defined in Section \ref{sec:prelim}. 
Given $n$ input samples, the $n$ features $\Phi(\x_i) \in \R^P$ yield  an $n \times P$ feature matrix ${\bm \Phi}$.

Our discussion will be based on the (empirical) {\bf Rademacher complexity}, which  shows up in generalization bounds \citep{bartlett2002rademacher}; see Appendix \ref{sec:margin_bounds} for a review.  
It measures how well $\cF$ correlates with random noise on the sample set $\cS$: 
\beq \label{def:empRad}
\widehat{\mathcal{R}}_{\cS}(\cF) = \E_{{\bm \sigma}\in \{\pm 1\}^n} \left[\sup_{f\in \cF} \frac{1}{n}  \sum_{i=1}^n \sigma_i f(\x_i)\right].
\eeq  
The Rademacher complexity depends on the size (or {\bf capacity})  of the class $\cF$. Constraints on the capacity, such as those induced by the implicit bias of the training algorithm, can reduce the  Rademacher complexity and lead to sharper generalization bounds.

A standard approach for controlling capacity  is in terms of the {\it norm} of the weight vector -- usually the $\ell_2$-norm. In general, given any invertible matrix $A \in \R^{P \times P}$, we may consider the norm $\|\w\|_{\! A} :=  \sqrt{\w^\top g_{\!A} \w}$ induced by the metric  $g_A = AA^{\!\top}$. Consider the (sub)classes of functions  induced by balls of given radius: 
\beq \label{eq:Aclasses}
\cF^A_{M_{\!A}} = \{f_\w \maps \x \mapsto \langle \w, \Phi(\x) \rangle \,\, | \,\, \|\w \|_{\! A} \leq M_{\!A}\}.
\eeq 
A direct extension of standard bounds for the Rademacher complexity (see Appendix~\ref{sec:proofthm1}) yields, 
\beq \label{eq:Abound}
\widehat{\mathcal{R}}_\cS(\cF^A_{M_{\!A}}) \leq  (M_{\!A}/n) \|A^{-1} {\bm \Phi}^{\!\top} \|_{\mathrm{F}} 
\eeq
where $\|A^{-1} {\bm \Phi}^{\!\top} \|_{\mathrm{F}}$ is the Froebenius norm of the {\it rescaled} feature matrix.\footnote{We also have $\| A^{-1} {\bm \Phi}^{\!\top}\|_\mathrm{F} \!=\! \sqrt{\Tr\K_{\!A}}$ 
in terms of the
(rescaled) kernel matrix $\K_{\!A} = {\bm \Phi} g_{\!A}^{-1} {\bm \Phi}^\top$.}

This freedom in the choice of rescaling matrix $A$ 
raises the question of which of the norms $\| \cdot \|_A$ provide meaningful measures of the model's capacity. Recent works  \citep{Belkin2018, Muthukumar:2020} pointed out that using $\ell_2$ norm is not coherently linked with generalization in practice. We discuss this issue in Appendix \ref{appendix:norm_capacity}, illustrating how meaningful norms critically depend on the geometry defined by the features.

\subsubsection{Feature Alignment as Implicit Regularization} \label{sec:superNat}

\begin{figure*}[t]
      \centering
      \begin{minipage}{0.5\linewidth}
      \hspace{-0.2cm}
      {\bf SuperNat} update ($\tilde{A}_0 = {\bm I}$, $\Phi_0 = \Phi$, $\K_0=\K$): 
      \vspace{0.2cm} 
         \begin{enumerate}[leftmargin=*] 
\item Perform gradient step $\widetilde{\w}_{t+1} \leftarrow \w_{t} + \delta \w_{\mbox{\!\tiny GD}}$
\item Find minimizer $\tilde{A}_{t+1}$ of $\|\delta \w_{\mbox{\!\tiny GD}}\|_{\!\tilde{A}} \|\tilde{A}^{-1} {\bm \Phi}_t^{\!\top} \|_{\mathrm{F}}$
\item Reparametrize: 
\[\w_{t+1} \leftarrow  \tilde{A}^\top_{t+1} \widetilde{\w}_{t+1}, \Phi_{t+1} \leftarrow \tilde{A}^{-1}_{t+1} \Phi_{t}\]
\end{enumerate}
 \end{minipage}
      \begin{minipage}{0.47\linewidth}
          \begin{figure}[H]
        \includegraphics[width=1\linewidth]{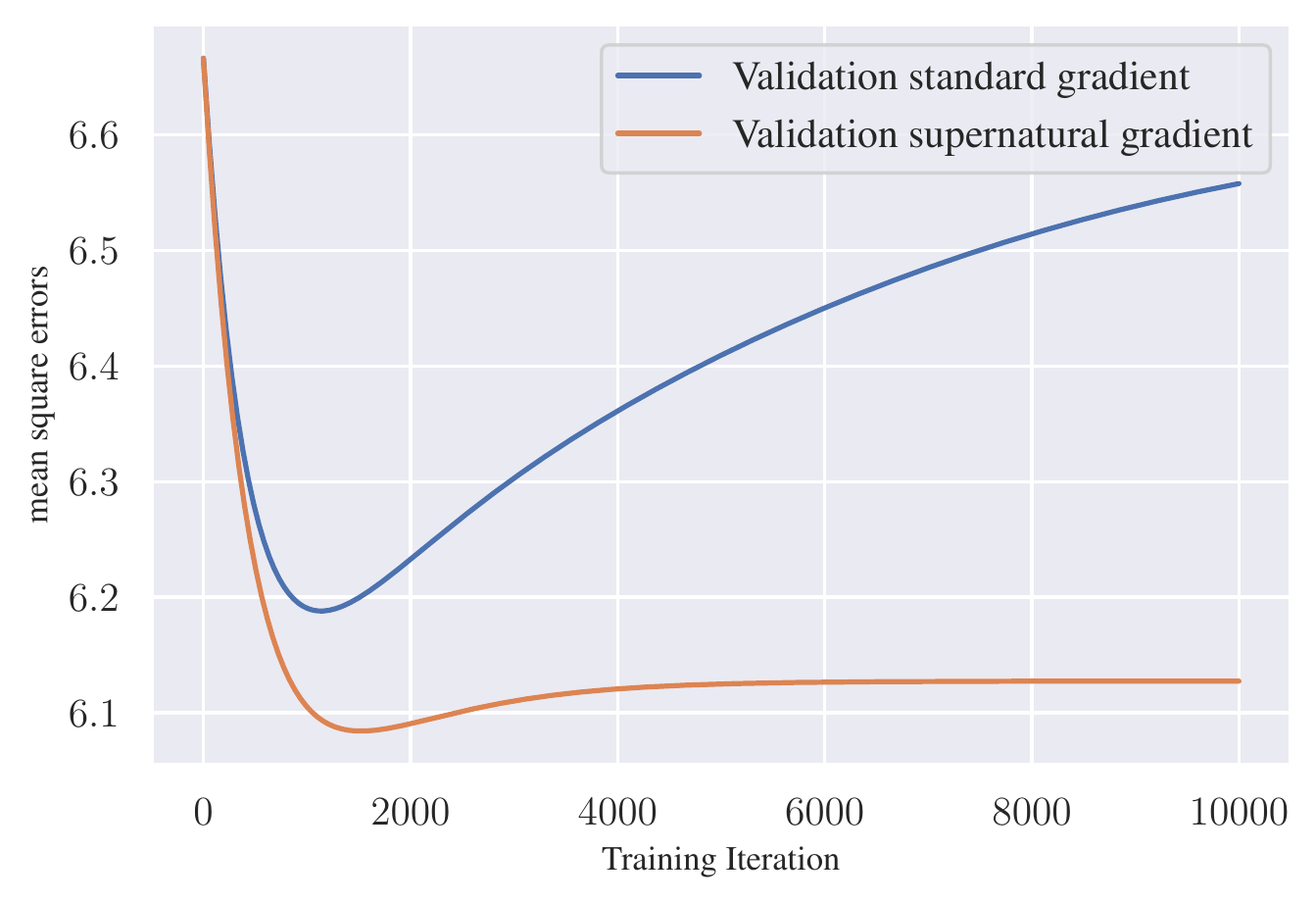}
          \end{figure}
      \end{minipage}
          \caption{\textbf{(left)} {\bf SuperNat} algorithm and \textbf{(right)} validation curves obtained with standard and {\bf SuperNat} gradient descent, on the noisy linear regression problem. At each iteration, {\bf SuperNat} identifies dominant features and stretches the kernel along them, thereby slowing down and eventually freezing the learning dynamics in the noise direction. This naturally yields better generalization than standard gradient descent on this problem.}
          \label{Fig:SuperNat}
  \end{figure*}

Here we describe a simple procedure making the geometry {\it adaptive} along optimization paths. The goal is to illustrate in a simple setting how feature alignment can impact complexity and generalization, in a way that mimics the behaviour of a non-linear dynamics.  The idea is to {\it learn} a rescaling metric  at each iteration of our algorithm, using a local version of the bounds (\ref{eq:Abound}).


\paragraph{Complexity of Learning Flows.} Since we are interested in functions $f_\w$ that  result from an iterative algorithm, we consider functions $f_\w =  \sum_t \delta f_{\w_t}$ written in terms of a sequence of updates\footnote{In order to not assume a specific upper bound on the number of iterations, we can think of the updates from an iterative algorithm as an infinite sequence $\{\delta \w_0, \cdots \delta \w_t, \cdots\}$ such that for some $T$, $\delta \w_t = 0$ for all $t>T$.} $\delta f_{\w_t}(\x) = \langle \delta \w_t, \Phi(\x) \rangle$ (we set $f_0=0$ to keep the notation simple), with {\it local} constraints on the parameter updates:
\beq \label{eq:dynAclasses}
\cF^{\! \bm A}_{\bm m} = \{f_\w \maps \x \mapsto {\textstyle \sum_t} \langle \delta \w_t, \Phi(\x) \rangle  \, | \,  \|\delta\w_t\|_{\!A_t} \leq m_t\}
\eeq
The result (\ref{eq:Abound}) extends as follows.
\begin{theo}[Complexity of Learning Flows]
\label{theo:comp_lf}
Given any sequences $\bm A$ and $\bm m$ of invertible matrices $A_t \in \R^{P\times P}$ and positive numbers $m_t >0$, we have the bound
\beq \label{theo:kern_newbound} 
\widehat{\mathcal{R}}_\cS(\cF^{\!\bm A}_{\bm m})  
\leq {\textstyle \sum_t} (m_t/n) \|A_t^{-1} {\bm \Phi}^{\!\top} \|_{\mathrm{F}}.
\eeq  
\end{theo} 
Note that, by linear reparametrization invariance $\w \mapsto  A^\top\w$, $\Phi \mapsto A^{-1}\Phi$,  the {\it same} result can be formulated in terms of the  sequence ${\bm \Phi} = \{\Phi_t\}_t$ of feature maps $\Phi_t = A_t^{-1} \Phi$. The function class (\ref{eq:dynAclasses}) can equivalently be written as 
\beq \label{eq:dynPhiclasses}
\cF^{\! \bm \Phi}_{\bm m} = \{f_\w \maps \x \mapsto {\textstyle \sum_t} \langle \tilde{\delta} \w_t, \Phi_t(\x) \rangle  \, | \,  \|\tilde{\delta}\w_t\|_{2} \leq m_t\}
\eeq
In this formulation, the result   (\ref{theo:kern_newbound}) reads:
\beq \label{kern_newbound} 
\widehat{\mathcal{R}}_\cS(\cF^{\!\bm \Phi}_{\bm m})  
\leq {\textstyle \sum_t} (m_t/n) \|{\bm \Phi}_t\|_\mathrm{F}.
\eeq 


\begin{figure*}[t]
\centering
	\includegraphics[width=\textwidth]{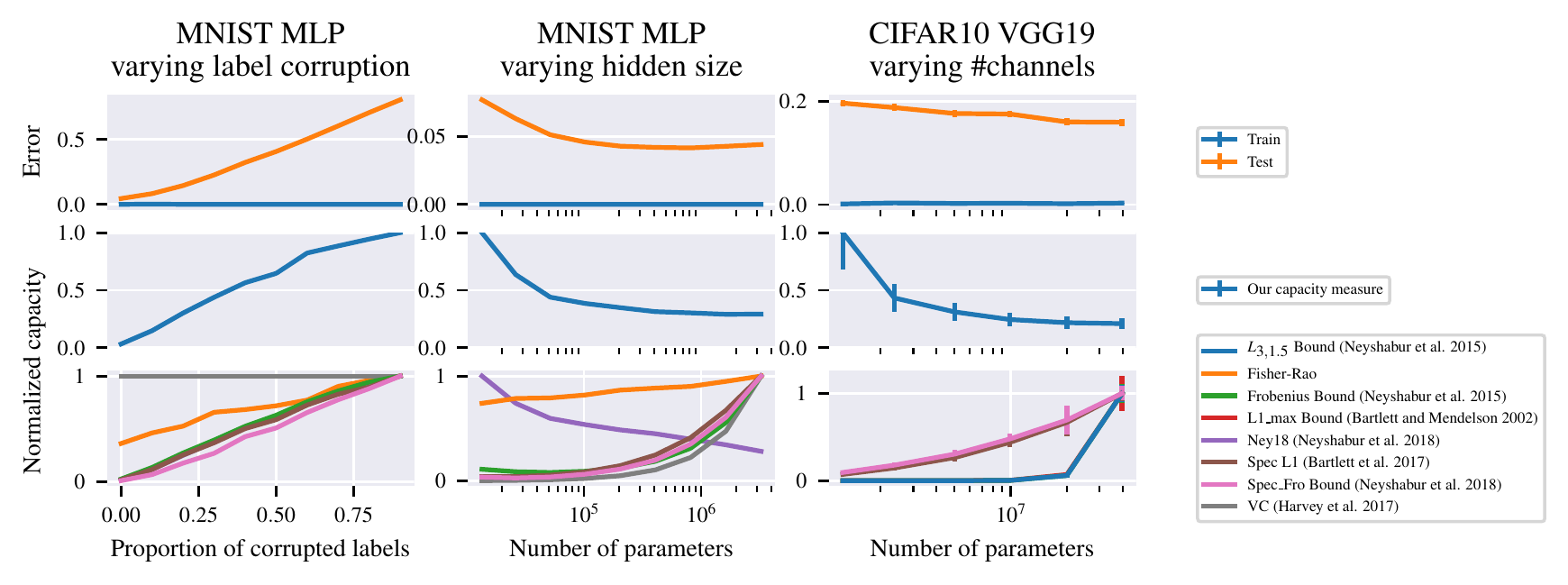}
	\caption{\small Complexity measures on MNIST with a 1 hidden layer MLP \textbf{(left)} as we increase the hidden layer size, \textbf{(center)} for a fixed hidden layer of 256 units as we increase label corruption and \textbf{(right)} for a VGG19 on CIFAR10 as we vary the number of channels. All networks are trained until cross-entropy  reaches $0.01$. Our proposed complexity measure and 
	the one by Neyshabur et al. 2018 are the only ones to correctly reflect the shape of the generalization gap in these settings.}
	\label{fig:complexity_measures}
\end{figure*}

\paragraph{Optimizing the Feature Scaling.} 
To obtain learning flows with lower complexity,  Thm.~\ref{theo:comp_lf} suggests modification of the algorithm to include, at each iteration $t$, a reparametrization step with a suitable matrix
$\tilde{A}_t$ giving a low contribution to the bound (\ref{theo:kern_newbound}). 
Applied to gradient descent (GD), this leads to a new update rule sketched in Fig.~\ref{Fig:SuperNat} (left), where the optimization in Step 2 is over a given class of reparametrization matrices. The successive reparametrizations yield a varying feature map $\Phi_t = A_t^{-1} \Phi$ where $A_t \!=\! \tilde{A}_0 \cdots \tilde{A}_t$.\footnote{Note that upon training a  non-linear model, the updates of the tangent features take the same form $\Phi_t = \tilde{A}_t^{-1}  \Phi_{t-1}$ as in Step 3 of {\bf SuperNat}, the difference being that $\tilde{A}_t$ is now a differential operator, e.g. at first order $\tilde{A}_t = \mathrm{Id} - \delta\w_t^\top \frac{\partial}{\partial \w_t}$.}
 
In the original representation $\Phi$, {\bf SuperNat} amounts to natural gradient descent \citep{amari1998natural} with respect to the local metric $g_{\!A_t}=A_t A_t^{\!\top}$. By construction, we also have $\delta f_{\w_t}(\x)=\langle \delta \w_{\mbox{\!\tiny GD}} , \Phi_t(\x)\rangle$  where $\delta \w_{\mbox{\!\tiny GD}}$ are  standard gradient descent updates in the linear model with feature map $\Phi_t$. 

As an example, let $\label{eq:featSVD} 
{\bm \Phi}= \sum_{j=1}^n \sqrt{\lambda_j} \u_j \v_j^\top$ be the SVD of the feature matrix.
We restrict to the class of matrices
\beq \label{eq:Aclass}
\tilde{A}_{\bm \nu} = \sum_{j=1}^n \sqrt{\nu_j} \v_j \v_j^\top + \mathrm{Id}_{\mathrm{span}\{\v\}^\perp}
\eeq 
labelled by weights $\nu_j >0, j =1,\cdots, n$. 
With such a class,  the action ${\bm \Phi}^\top_t \to A_{{\bm \nu}}^{-1} {\bm \Phi}^\top_{t}$ merely rescales the singular values $\lambda_{jt} \to \lambda_{jt} / \nu_j$, leaving  the singular vectors unchanged. We work with gradient descent w.r.t a cost function $L$, so that $\delta {\w_{\!\mbox{\tiny GD}}} = - \eta \nabla_{\!\w} L$.
\begin{prop} \label{prop:optimal_supernorm} Any minimizer in Step 2 of {\bf SuperNat} over matrices $\tilde{A}_{\bm \nu}$ in the  class  (\ref{eq:Aclass}), takes the form  
\beq  \label{eq:nusol2} 
\nu_{jt}^\ast = \kappa \frac{1}{|\u_j^{\!\top} \nabla_{\!\f_\w} L|}
\eeq
where $\nabla_{\!\f_\w}$ denotes the gradient w.r.t the sample outputs  $f_\w := [f_\w(\x_1), \cdots f_\w(\x_n)]^\top$, for some constant $\kappa>0$. 
\end{prop} 
In this context, this yields the following update rule, up to isotropic rescaling, for the singular values of ${\bm \Phi}_t$: 
\beq 
\lambda_{j(t+1)} = |\u_j^{\!\top} \nabla_{\f_\w} L| \lambda_{jt}.
\eeq 
 In this illustrative setting, we see how the feature map (or kernel) adapts to the task, by stretching (resp. contracting) its geometry in directions $\u_j$ along which  the residual $\nabla_{\f_\w} L$ has large (resp. small) components. Intuitively, if a large component $|\u_j^{\!\top} \nabla_{\f_\w} L|$ corresponds to signal and a small one $|\u_k^{\!\top} \nabla_{\f_\w} L|$ corresponds to noise, then the ratio $\lambda_{jt}/\lambda_{kt}$ of singular values gets rescaled by the signal-to-noise ratio,  thereby increasing the alignment of the learned features to the signal. 

As a proof of concept, we consider the following regression setup. We consider  a linear model with Gaussian features  $\Phi = [\varphi,\varphi_{\mathrm{\footnotesize noise}}] \in \R^{d+1}$ where $\varphi \sim \cN(0,1)$  and $\varphi_{\mathrm{\footnotesize noise}} \sim \cN(0, \frac{1}{d} I_{d})$. Given $n$ input samples, the $n$ features $\Phi(\x_i)$ yield 
${\bm \varphi} \in \R^n$ and ${\bm \varphi}_{\mathrm{\footnotesize noise}} \in \R^{ n\times d}$.  We assume the label vector takes the form  ${\bm y} = {\bm \varphi} +  P_{\mathrm{\footnotesize noise}}({\bm \epsilon}) $, where Gaussian noise ${\bm \epsilon} \sim \cN(0, \sigma^2 I_n)$ is projected  onto the noise features through $P_{\mathrm{\footnotesize noise}} = {\bm \varphi}_{\mathrm{\footnotesize noise}} {\bm \varphi}_{\mathrm{\footnotesize noise}}^\top$.    The model is trained by gradient descent of the mean squared loss and its {\bf SuperNat} variant, where Step 2 uses the analytical solution of Proposition \ref{prop:optimal_supernorm}. We set $d=10, \sigma^2 = 0.1$ and use $n=50$ training points.

Fig \ref{Fig:SuperNat} (right) shows test error obtained with standard and {\bf SuperNat} gradient descent on this problem. At each iteration, {\bf SuperNat} identifies dominant features (feature selection, here $\varphi$) and stretches the metric along them, thereby slowing down and eventually freezing the dynamics in the orthogonal (noise) directions (compression). The working hypothesis in this paper, supported by the observations of Section \ref{sec:nonlin}, is that for neural networks, such a (tangent) feature alignment is dynamically induced as an effect of non-linearity.
 
\subsection{A New Complexity Measure for Neural Networks}

 Equ.~(\ref{kern_newbound}) provides a bound of the Rademacher complexity for the function classes (\ref{eq:dynAclasses})  specified by a {\it fixed} sequence of feature maps (see Appendix  \ref{app:bounds_multiclass} for a generalization to the multiclass setting). By extrapolation to the case of non-deterministic sequences of feature maps, we propose using \beq \label{eq:complexity}
  \cC(f_\w) = \sum_t \|\delta \w_t\|_2 \|{\bm \Phi_t}\|_{\mathrm{F}}
  \eeq
  as a heuristic measure of complexity for neural networks, where ${\bm \Phi}_t$ is the learned tangent feature matrix\footnote{In terms of tangent kernels,  $\|{\bm\Phi_t}\|_{\mathrm{F}} = \sqrt{\Tr \K_t}$ where $\K_t$ is the tangent kernel (Gram) matrix.}  at training iteration $t$, and $\|\delta \w_t\|_2$ is the norm of the SGD update. Following a standard protocol for studying complexity measures, \citep[e.g.,][]{neyshabur2017exploring},
  Fig.~\ref{fig:complexity_measures}  shows its behaviour for MLP on MNIST and VGG19 on CIFAR10 trained with cross entropy loss, with \textbf{(left)} fixed architecture and varying level of corruption in the labels and \textbf{(right)} varying hidden layer size/number of channels up to 4 millions parameters, against other capacity measures  proposed in the recent literature. We observe that it correctly reflects the shape of the generalization gap.

\section{Related Work}

{\bf Role of Feature Geometry in Linear Models.} Analysis of the relation between capacity and feature geometry can be traced back to early work on kernel methods~\citep{Scholkopf:1999}, which lead to data-dependent error bounds in terms of the eigenvalues of the kernel Gram matrix~\citep{Scholkopf:1999-2}.

Recently, new analysis of minimum norm interpolators and max margin solutions for overparametrized linear models emphasize the key role of feature geometry, and specifically feature anisotropy, in the generalization performance~\citep{BenignOverfit, Muthukumar:2019, Muthukumar:2020,Ward2020}. 
Feature anisotropy combined to a high predictive power of the dominant features is the condition for a high centered alignment between kernel and class labels.  In the context of neural networks, our results highlight the role of the non linear training dynamics  in favouring such conditions.  

{\bf Generalization Measures.} 
There has been a large body of work on complexity/generalization measures for neural networks~\citep[see,][and references therein]{Jiang*2020Fantastic}, some of which theoretically motivated by norm or margin based bounds~\citep[e.g.,][]{neyshabur2019, Bartlett:2017}. \citet{liang19a} proposed using the Fisher-Rao norm of the solution as a geometrically invariant  complexity measure. 
By contrast, our approach to measuring complexity takes into account the geometry along the whole optimization trajectories. Since the geometry we consider is defined through the gradient second moments, our perspectice is closely related  to the notions of stiffness~\citep{fort2019stiffness} and coherent gradients~\citep{Chatterjee2020Coherent}.


{\bf Dynamics of Tangent Kernels}.  Several recent works  investigated the 'feature learning' regime where neural tangent kernels evolve during training ~\citep{Geiger2019, Kernel_rich}. 
Independent concurrent works highlight alignment and compression phenomena similar to the one we study here \citep{NTKalignment, Paccolat2020}. We offer various complementary empirical insights, and frame the alignment mechanism from the point of view of implicit regularization.

\section{Conclusion}

Through experiments with modern architectures, we highlighted an effect of dynamical alignment of the neural tangent features and their kernel along a small number of task-dependent  directions during training, reflected by an early drop of the effective rank and an increasing similarity with the class labels, as measured by centered kernel alignment. We interpret this effect as  a  combined mechanism of   feature selection and model compression of around dominant features. 

Drawing upon intuitions from linear models, we argued that such a dynamical alignment acts as implicit regularizer. By extrapolating a new analysis of Rademacher complexity bounds for linear models, we also proposed a complexity measure that captures this phenomenon, and showed that it correlates with the generalization gap when varying the number of parameters, and when increasing the proportion of corrupted labels.

The results of this paper open several avenues for further investigation. 
The type of complexity measure we propose suggests new principled ways to design algorithms that learn the geometry in which to perform gradient descent  \citep{SrebroMirrorDescent2011,neyshabur2017norm}. Whether a procedure such as {\bf SuperNat}
can produce meaningful practical results for neural networks remains to be seen. 

One of the consequences one can expect from the alignment effects highlighted here is to bias learning towards explaining most of the data with a small number of highly predictive features. While this feature selection ability  might explain in part the performance of neural networks on a range of supervised tasks, it may also make them brittle under spurious correlation 
 \citep[e.g.,][]{Liang2020} and underpin  their notorious weakness to generalize out-of-distribution \citep[e.g.,][]{Geirhos_2020}. Resolving this tension is an important challenge towards building more robust models. 

\ackaccepted

We thank X. Y Lu and V. Thomas for collaboration at an early stage of this project,   A. Sordoni for insightful discussions,  G. Gidel, A. Mitra and M. Pezeshki for helpful feedback. This research was partially supported by the Canada CIFAR AI Chair Program (held at Mila); 
by NSERC through the Discovery Grants RGPIN-2017-06936 (S.LJ) and RGPIN-2018-04821 (G.L) 
and an Alexander Graham Bell Canada Graduate Scholarship (CGS D) award (A.B); by
FRQNT Young Investigator Startup Program  2019- NC-253251 (G.L); 
and by a Google Focused Research award (S.LJ).
S.LJ and P.V are CIFAR Associate Fellows in the Learning in Machines \& Brains program.

\newpage

\bibliography{references}
\bibliographystyle{iclr2021_conference} 

\newpage 

\appendix
\AtAppendix{\counterwithin{lem}{section}}
\AtAppendix{\counterwithin{thm}{section}}
\AtAppendix{\counterwithin{prop}{section}}

\onecolumn

\begin{center} 
{\large \bf APPENDICES: Implicit Regularization via Neural Feature Alignment}
\end{center}


\section{Tangent Features and Geometry} 
\label{appendix:geom}

We describe in more formal detail some of the notions  introduced in Section \ref{sec:prelim} of the paper. 
We will consider general classes of vector-valued predictors: 
\beq \label{eq:model}
\cF = \{f_\w \maps \cX  \to \R^c\,\, | \,\,  \w \in \cW\},
\eeq 
where the parameter space $\cW$ is a finite dimensional manifold of dimension $P$ (typically $\R^P$). For multiclass classification, $f_\w$ outputs a score $f_\w(\x)[y]$ for each class $y \in \{1 \cdots c\}$.  Each function can also be viewed as a scalar function on $\cX \times \cY$ where $\cY = \{1\cdots c\}$ is the set of classes. 


\subsection{Metric} 

We assume that $\w \to f_\w$ is a smooth mapping from $\cW$ to $L^2(\rho, \R^c)$, where $\rho$ is some input data distribution.
The inclusion $\cF \subset L^2(\rho, \R^c)$ equips $\cF$ with the $L^2$ scalar product and corresponding norm:
\beq \label{eq:scalar_prod}
\langle f, g\rangle_{\rho}:= \E_{\x\sim \rho} [f(\x)^\top g(\x)], \qquad \|f\|_\rho := \sqrt{\langle f, f\rangle_\rho}
\eeq 
The parameter space $\cW$ inherits a {\bf metric tensor} $g_\w$  by pull-back of the scalar product $\langle f, g \rangle_\rho$ on $\cF$.  
That is, given 
$\bzeta, \bxi \in \cT_{\w}\cW \cong \R^P$  on the tangent space at $\w$ \citep{lang2012fundamentals}, 
\beq \label{eq:metrictensor}
g_\w(\bzeta, \bxi)  = \langle \partial_{\bzeta} f_\w, \partial_{\bxi} f_\w \rangle_\rho  
\eeq 
where $\partial_{\bzeta} f_\w = \langle df_\w, \bzeta \rangle$ is the directional derivative in the direction of $\bzeta$. Concretely, in a given basis of $\R^P$, the metric is represented by the matrix of gradient second moments: 
\beq \label{eq:basis_metrictensor}
(g_\w)_{pq} = \E_{\x \sim \rho}\left[\left(\frac{\partial f_\w(\x)}{\partial w_p}\right)^\top \frac{\partial f_\w(\x)}{\partial w_q}\right]
\eeq
where $w_p, \, p=1, \cdots, P$ are the parameter coordinates. The metric shows up by spelling out the line element $ds^2 :=\|d f_\w\|_\rho^2$, since we have,
\beq 
\|d f_\w\|_\rho^2 = \sum_{p,q=1}^P \langle \frac{\partial f_\w}{\partial w_p} dw_p, \frac{\partial f_\w}{\partial w_q} dw_q \rangle_\rho    
= \sum_{p,q=1}^P (g_\w)_{pq}\, dw_p dw_q
\eeq

\subsection{Tangent Kernels}

This geometry has a dual description in function space in terms of {\it kernels}.
The idea is to view the differential of the mapping $\w \to f_\w$ at each $\w$ as a map $df_\w \maps \cX \times \cY \to \cT^\ast_\w\cW \cong \R^p$  defining (joined) features in the (co)tangent space. In a given basis, this yields the {\bf tangent features} given by the function derivatives w.r.t the parameters, \beq \Phi_{\!w_p}(\x)[y] := \frac{\partial f_\w(\x)[y]}{\partial w_p}
\eeq
The tangent feature map $\Phi_{\!\w}$ can be viewed as a function mapping each pair  $(\x, y)$ to a vector in $\R^P$.  It defines the so-called {\bf tangent kernel} \citep{NTK} through the Euclidean dot product $\langle \cdot , \cdot \rangle$ in $\R^P$:  
\beq \label{eq:tangentKer}
k_\w(\x, y; \tilde{\x}, y') = \langle 
\Phi_{\!\w}(\x)[y], \Phi_{\!\w}(\tilde{\x})[y'] \rangle = 
\sum_{p=1}^P \Phi_{\! w_p}(\x)[y] \Phi_{\!w_p}(\tilde{\x})[y']
\eeq
It induces an integral operator on $L^2(\rho, \R^c)$ acting as
\beq \label{sec_eq:kaction}
(k_\w \act f)(\x)[y] = \langle k_\w(\x, y ; \, \cdot \,), f \rangle
\eeq 
The metric tensor (\ref{eq:basis_metrictensor}) is expressed in terms of the tangent features as $(g_\w)_{pq} = \langle \Phi_{w_p},\Phi_{w_q}\rangle_\rho$.

\subsection{Spectral Decomposition} 
\label{appsec:spectral}

The local metric tensor (as symmetric $P \times P$ matrix) and tangent kernel (as rank $P$ integral operator) share the same spectrum.  More generally, let 
\beq \label{sec_app:gspec}
g_\w \!=\! \sum_{j=1}^{P} \lambda_{\w j} \v_{\w j} \v_{\w j}^{\!\top} 
\eeq 
be the eigenvalue decomposition of the  positive (semi-)definite  symmetric matrix (\ref{eq:basis_metrictensor}), where $\v_{\w j}^\top \v_{\w j'} = \delta_{j j'}$. Assuming non-degeneracy, i.e $\lambda_{\w j} > 0$, let $u_{\w j}, j \in \{1\cdots P\}$ be the functions in  $L^2(\rho, \R^c)$ defined as:
\beq  \label{appeq:pca_comp}
u_{\w j}(\x)[y] = \frac{1}{\sqrt{\lambda_{\w j}}} \v_{\w j}^{\top} \Phi_\w(\x)[y]
\eeq 
The following result holds. 
\begin{prop}[Spectral decomposition]
\label{sec_app:gkspec}
The functions $(u_{j \w})_{j=1}^P$ form an orthonormal family in $L^2(\rho, \R^c)$. They are eigenfunctions of the tangent kernel as an integral operator, which admits the spectral decomposition:
\beq  \label{app_sec:kspec}
k_\w(\x, y; \tilde{\x}, y') = \sum_{j=1}^P \lambda_{\w j}\,  u_{\w j}(\x)[y] \, u_{\w j}(\tilde{\x})[y'] 
\eeq 
In particular metric tensor and tangent kernels  share the same spectrum. 
\end{prop}
\begin{proof} We first show orthonormality, i.e $\langle u_{\w j} u_{\w j'} \rangle_\rho = \delta_{j j'}$. We have indeed,
\begin{align}
\langle u_{\w j} , u_{\w j'} \rangle_\rho &= \frac{1}{\sqrt{\lambda_{\w j}\lambda_{\w j'}}} \sum_{p,q=1}^P (\v_{\w j})_p  (\v_{\w j})_{q} \langle \Phi_{w_p},\Phi_{w_q}\rangle_\rho \\ 
&= \frac{1}{\sqrt{\lambda_{\w j}\lambda_{\w j'}}} \v_{\w j}^\top \, g_\w \, \v_{\w j'} \\
&= \frac{1}{\lambda_{\w j}} \lambda_{\w j} \delta_{jj'} \\ 
&= \delta_{jj'} 
\end{align}
where we used the definition of the matrix $(g_\w)_{pq}$ and its eigenvalue decomposition. Next, using the action (\ref{sec_eq:kaction}) of the tangent kernel, we prove that the  functions $u_{\w j}$ defined in (\ref{appeq:pca_comp}) is an eigenfunction with eigenvalue $\lambda_{\w j}$:
\begin{align}
(k_\w \act u_{\w j} )(\x)[y] &=
\sum_{p=1}^P \Phi_{\! w_p}(\x)[y] \langle \Phi_{\! w_p}, u_{\w j} \rangle_{\rho} \\
&=\frac{1}{\sqrt{\lambda_{\w j}}} \sum_{p,q=1}^P (\v_{\w j})_q\, \Phi_{\! w_p}(\x)[y]  \langle \Phi_{\! w_p},  \Phi_{\! w_q} \rangle  \\ 
&=\frac{1}{\sqrt{\lambda_{\w j}}}
\v_{\w j}^\top \, g_\w \, \Phi_{\! \w}(\x)[y] \\
&= \frac{1}{\sqrt{\lambda_{\w j}}} \, (\lambda_{\w j} \v_{\w j}^\top) \, \Phi_{\! \w}(\x)[y]\\
&= \lambda_{\w j} \frac{1}{\sqrt{\lambda_{\w j}}}  \v_{\w j}^\top \, \Phi_{\! \w}(\x)[y] \\
&= \lambda_{\w j} \, u_{\w j}
\end{align} 
Inserting the resolution of unity $\mathrm{Id}_P = \sum_{j=1}^P \v_{\w j} \v_{\w j}^\top$ in the expression (\ref{eq:tangentKer}) of the tangent kernel directly yields the spectral decomposition (\ref{app_sec:kspec}). 
\end{proof} 

\subsection{Sampled Versions}
\label{secapp:sampled}
Given $n$ input samples $\x_1, \cdots \x_n$, any function $f \maps \cX  \to \R^c$ yields a vector  $\f \in \R^{nc}$  
obtained by concatenating the outputs $f(\x_i) \in \R^c$ of the $n$ input samples $\x_i$. 
The sample output scores $f_\w(\x_i)[y]$ thus yields $\f_\w \in \R^{nc}$; and the tangent features $\Phi_{\!w_p}(\x_i)[y]$ are represented as a $nc \times P$ matrix ${\bm \Phi}_{\!\w}$. Using this notation, (\ref{eq:basis_metrictensor}) and (\ref{eq:tangentKer}) yield the sample covariance $P \times P$ matrix and kernel (Gram) $nc \times nc$ matrix:  
\beq \label{eq:empMat}
{\bm G}_{\w}= {{\bm \Phi}_{\!\w}}^{\hspace{-0.25cm} \top} {\bm \Phi}_{\!\w}, \quad {\bm K}_{\!\w} ={\bm \Phi}_{\!\w} {{\bm \Phi}_{\!\w}}^{\hspace{-0.25cm} \top}
\eeq 
The eigenvalue decompositions of $\bm G_{\w}$ and $\bm K_{\w}$ follow from the (SVD) of ${\bm \Phi}_{\!\w}$: assuming $P > nc$, we can write this SVD by indexing the singular values by a pair $J=(i, y)$ with $i=1,\cdots n$ and $y=1\cdots c$ as 
\beq \label{eqapp:feat_svd}
{\bm \Phi}_{\!\w} = \sum_{J=1}^{nc}  \sqrt{{\hat \lambda}_{\w J}} \hat{\u}_{\w J}  \hat{\v}_{\w J}^{\!\top}
\eeq 
Such decompositions summarize the predominant directions both in parameter and feature space, in the neighborhood of $\w$:
a small variation $\delta \w$ induces the first order variation  $\delta \f_\w$ of the function,
\beq 
\delta \f_\w := \Phi_\w \delta \w = \sum_{J=1}^{nc} \sqrt{{\hat \lambda}_{\w J}} (\hat{\v}_{\w J}^T \delta \w) \hat{\u}_{\w J}
\eeq
Fig.~\ref{fig:perturbednet} illustrates this `hierarchy'
for a VGG11 network \citep{simonyan2014very} trained for 10 epoches on CIFAR10 \citep{krizhevsky2009learning}. We observe that perturbations in most directions have almost no effect, except in those aligned with the top singular vectors. This is reflected by a strong anisotropy of the tangent kernel spectrum. Recent analytical results for wide random neural networks also point to such a pathological structure of the spectrum \citep{AmariPathological, AmariUNiversal}. 
\begin{figure}
	\centering
	\includegraphics[width=0.49\textwidth]{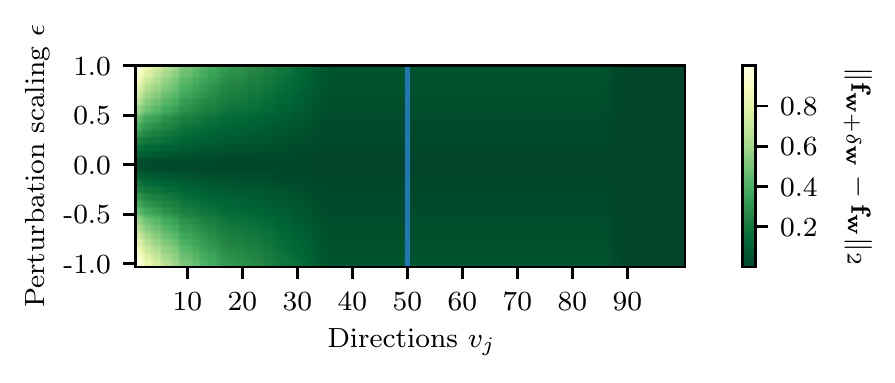}
	\caption{\label{fig:perturbednet} Variations of $\f_\w$ (evaluated on a test set) when perturbing the parameters in the directions given by the right singular vectors of the Jacobian (first 50 directions) or in randomly sampled directions (last 50 directions) on a VGG11 network trained for 10 epochs on CIFAR10. We observe that perturbations in most directions have almost no effect, except in those aligned with the top singular vectors.}
\end{figure}

\subsection{Spectral Bias} 
\label{appendix:spec_bias}

\subsubsection{Proof of Lemma \ref{lemma:PCAlinear}} 

We consider parameter updates $\delta {\w_{\!\mbox{\tiny GD}}} := - \eta \nabla_{\!\w} L$ 
for gradient descent w.r.t a loss $L:=L(\f_\w)$, which is a function of the vector $\f_\w \in \R^{nc}$ of sample output scores. We reformulate Lemma \ref{lemma:PCAlinear}, extended to the multiclass setting. 
\begin{prop}[Lemma \ref{lemma:PCAlinear} restated]
\label{app_lemma:PCAlinear} 
The gradient descent function updates in first order Taylor approximation,   
$\delta f_{{\!\mbox{\tiny GD}}}(\x)[y] := \langle \delta {\w_{\!\mbox{\tiny GD}}}, \Phi_\w(\x)[y]\rangle$,  decompose as,
\beq 
\label{eqapp:fct_updates}
\delta f_{\mbox{\tiny GD}}(\x)[y] = \sum_{j=1}^P  \delta f_j \, u_{\w j}(\x)[y], \qquad \delta f_j = -\eta \lambda_{\w j} (\u_{\w j}^{\!\top} \nabla_{\!\f_\w} L)   \eeq 
where $u_{\w j}$ are the eigenfunctions (\ref{appeq:pca_comp}) of the tangent kernel and $\u_{\w j} \in \R^{nc}$ are their corresponding sample vector. 
\end{prop}
\begin{proof} 
Inserting the resolution of unity $\mathrm{Id}_P = \sum_{j=1}^P \v_{\w j} \v_{\w j}^\top$ in the expression for $\delta f_{\mbox{\tiny GD}}$ yields
\begin{align} \label{appeq:fctupdate_spell}
\delta f_{{\!\mbox{\tiny GD}}}(\x)[y] &= \sum_{j=1}^P (\v_{\w j}^\top \delta {\w_{\!\mbox{\tiny GD}}}) \, \v_{\w j}^\top \Phi_\w(\x)[y] \\ 
&= \sum_{j=1}^P \sqrt{\lambda_{\w j}} (\v_{\w j}^\top \delta {\w_{\!\mbox{\tiny GD}}}) \, u_{\w j}(\x)[y]
\end{align} 
Next, by the chain rule $\nabla_{\!\w} L = {{\bm \Phi}_{\!\w}}^{\hspace{-0.25cm} \top}  \nabla_{\!\f_\w} L$, so we can spell out:
\beq \delta {\w_{\!\mbox{\tiny GD}}} = - \eta \sum_{j=1}^P \sqrt{\lambda_{\w j}} (\u_{\w j}^{\!\top} \nabla_{\!\f_\w} L) \, \v_{\w j},
\eeq 
which implies that $(\v_{\w j}^\top \delta {\w_{\!\mbox{\tiny GD}}}) = \sqrt{\lambda_{\w j}} (\u_{\w j}^{\!\top} \nabla_{\!\f_\w} L)$. Substituting in (\ref{appeq:fctupdate_spell}) gives the desired result.
\end{proof}

The decomposition (\ref{eqapp:fct_updates}) has a {\it sampled} version in terms of tangent feature and kernel matrices. Using the notation of SVD (\ref{eqapp:feat_svd}), let ${\hat \lambda}_{\w J}, \hat{\u}_{\w j}$ and $\hat{\v}_{\w j}$  be correspond to the (non-zero) eigenvalues and eigenvectors of the sample covariance and kernel (\ref{eq:empMat}). We consider the tangent kernel {\bf principal components}, defined as the functions 
\beq  \label{seceq:approx_ev}
\hat{u}_{\w J}(\x)[y] = \frac{1}{\sqrt{\lambda_{\w J}}} \langle \hat{\v}_{\w J}, \Phi_\w(\x)[y]\rangle,
\eeq 
which form an orthonormal family for the in-sample scalar product
$\langle f, g \rangle_{\mathrm{\footnotesize{in}}} = \sum_{i=1}^n f(\x_i) g(\x_i)$ and approximate the true kernel eigenfunctions (\ref{appeq:pca_comp})  \citep[e.g.,][]{Bengio-PCA, braun2005spectral}.  One can easily check from  (\ref{eqapp:feat_svd}) that the vector $\hat{\u}_{\w J} \in \R^{nc}$ of sample outputs $\hat{u}(\x_i)[y]$ coincides with the $J$-th eigenvector of the tangent kernel matrix.  
\begin{prop}[Sampled version of Prop \ref{app_lemma:PCAlinear}]
The gradient descent function updates in first order Taylor approximation,   
$\delta f_{{\!\mbox{\tiny GD}}}(\x)[y] := \langle \delta {\w_{\!\mbox{\tiny GD}}}, \Phi_\w(\x)[y]\rangle$  decompose as,
\beq 
\label{eqapp:fct_updates}
\delta f_{\mbox{\tiny GD}}(\x)[y] = \sum_{j=1}^{nc}   \delta f_J \, \hat{u}_{\w J}(\x)[y], \qquad \delta f_J = -\eta \hat{\lambda}_{\w J} (\hat{\u}_{\w J}^{\!\top} \nabla_{\!\f_\w} L)   \eeq 
in terms of the principal components (\ref{seceq:approx_ev}) of the tangent kernel. 
\end{prop}
\begin{proof} Same proof as for the previous Proposition, using the resolution of unity $\mathrm{Id}_{nc} = \sum_{J=1}^{nc} \hat{\v}_{\w J} \hat{\v}_{\w J}^\top$. 
\end{proof}

\subsubsection{The Case of Linear Regression}
\label{appendix:spec_bias_linear}

The previous Proposition gives a `local' version of a classic decomposition of the training dynamics in linear regression \citep[e.g.,][]{advani2017high}).  In such a setting, $f_\w = \langle \w , \Phi(\x) \rangle$ are linearly parametrized scalar functions ($c=1$) and 
$L =\frac12 \|\f_\w - {\bm y}\|^2$. We denote by ${\bm \Phi} = \sum_{j=1}^n \hat{\lambda}_j \hat{\u}_j \hat{\v}_j^\top$ the $n\times P$ feature matrix and its SVD.  
\begin{prop} Gradient descent of the squared loss yields the function iterates,
\beq 
\label{eq:gdfctiterates}
f_{\w_t} = f_{\w^\ast} + (\mathrm{Id} -\eta K)^t 
(f_{\w_0} - f_{\w^\ast})
\eeq
where $\mbox{Id}$ is the identity map and $K$ is the operator acting on functions as
$
(K\act f)(\x)= \sum_{i=1}^n k(\x, \x_i) f(\x_i)
$ 
in terms of the kernel $k(\x, \tilde{\x}) = \langle \Phi(\x), \Phi(\tilde{\x}) \rangle$.
\end{prop}
\begin{proof}
The updates $\delta {\w_{\!\mbox{\tiny GD}}} := - \eta \nabla_{\!\w} L$  induce the (exact) functional updates $\delta f_{\mbox{\tiny GD}} \!=\! f_{\w_{t+1}} - f_{\w_t}$ given by
\beq 
\delta f_{\mbox{\tiny GD}}(\x) = -\eta \sum_{i=1}^n k(\x, \x_i)(f_{\w_t}(\x_i) - \y_i)
\eeq 
Substituting $\y_i = f_{\w^\ast}(\x_i)$ gives $f_{\w_{t+1}} - f_{\w^\ast} = (\mbox{id} -\eta K)(f_{\w_t} - f_{\w^\ast})$. Equ.~\ref{eq:gdfctiterates} follows by induction. 
\end{proof}
\begin{lem}
The kernel principal components $\hat{u}_j(\x) = \frac{1}{\sqrt{\hat{\lambda}_j}} \langle \hat{\v}_j, \Phi_\w(\x)\rangle$ are eigenfunctions of the operator $K$ with corresponding eigenvalues $\hat{\lambda}_j$. 
\end{lem}
\begin{proof} 
By inserting  $Id_n = \sum_j \hat{\v}_j \hat{\v}_j^\top$ in the expression of the kernel, one can write $k(\x, \x_i) = \sum_{j=1}^n \hat{u}_j(\x) \hat{u}_j(\x_i)$. Subsituting in the definition of $K$ and using the orthonormality of $\hat{u}_j$ for the in-sample scalar product yield $K \act \hat{u}_j = \hat{\lambda}_j \hat{u}_j$.
\end{proof} 
Together with(\ref{eq:gdfctiterates}), this directly leads to the decoupling of the training dynamics in the basis of kernel principal components. 
\begin{prop}[Spectral Bias for Linear Regression]
\label{lemma:eigbasis} 
By initializing $\w_0 = {\bm \Phi}^{\!\top} {\bm \alpha}_0$ in the span of the features, the function iterates in (\ref{eq:gdfctiterates})  uniquely decompose as, 
\beq 
\label{eq:modedynamics}
f_{\w_t}(\x) = \sum_{j=1}^n f_{jt} \hat{u}_j(\x), \quad f_{jt} =  f_j^\ast + (1-\eta \lambda_j)^t \, (f_{j0}-f_j^\ast)
\eeq
where $f_j^\ast$ are the coefficients of the (mininum $\ell_2$-norm) interpolating solution.
\end{prop}
This standard result shows how each independant mode labelled by $j$ has its own linear convergence rate 
For example setting $\eta = 1/\lambda_1$, this gives $f_{jt} - f_j^\ast \propto e^{-t/\tau_j}$, where $\tau_j = -\log(1-\frac{\lambda_j}{\lambda_1})$ is the time constant (number of iterations) for the mode $j$. Top modes $f_j^\ast$ of the target function are learned faster than low modes. 

In linearized regimes where deep learning reduces to kernel regression \citep{NTK, du2018gradient, Allen-ZhuNTK}, one can dwell further the  nature of such a bias by analyzing the eigenfunctions of the neural tangent kernel \citep[e.g.,][]{Yang2019}. 
As a simple example, for a randomly initialized  MLP on  1D uniform data,  Fig.~\ref{fig:NTK_freq} shows the Fourier decomposition of such eigenfunctions, ranked in nonincreasing order of the eigenvalues. We observe that eigenfunctions with increasing index $j$ (hence decreasing eigenvalues)   correspond to modes with increasing Fourier frequency, with a remarkable alignment with Fourier modes for the first half of the spectrum. This in line with observations \citep[e.g.,][]{SpectralBias} that deep networks tend to prioritize learning low frequency modes during training.

\begin{figure}[t]
    \includegraphics[width=0.99\textwidth]{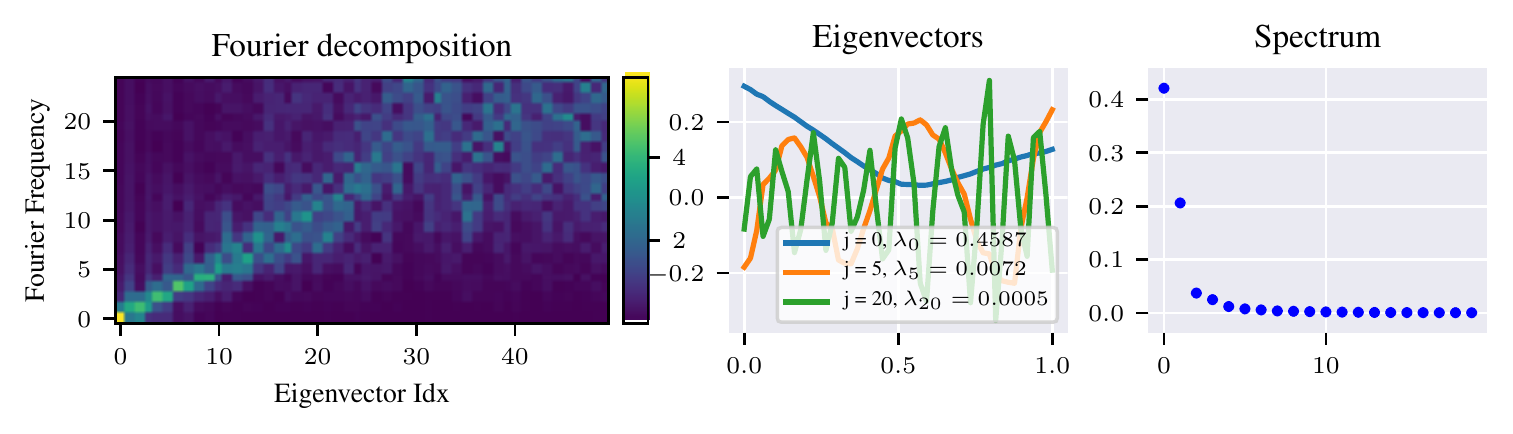}
\caption{\small Eigendecomposition of the tangent kernel matrix of a random 6-layer deep 256-unit wide MLP on 1D uniform data ($50$ equally spaced points in $[0,1]$). {\bf (left)}
Fourier decomposition ($y$-axis for frequency, colorbar for magnitude) of  each eigenvector ($x$-axis), ranked in nonincreasing order of the eigenvalues. We observe that eigenvectors with increasing index $j$ (hence decreasing eigenvalues) correspond to modes with increasing Fourier frequency. {\bf (middle)} Plot of the $j$-th eigenvectors with $j\in \{0, 5, 20\}$ and {\bf (right)} distribution of eigenvalues. We note the fast decay (e.g  $\lambda_{10}/\lambda_1  \approx 4\text{\textperthousand}$).}
\label{fig:NTK_freq}
 \end{figure}
 
 \section{Complexity Bounds}
\label{secapp:comp_bounds}

In this section, we spell out details and proofs for the content of Section \ref{sec:complexity_measure}. 
\subsection{Rademacher Complexity} 

Given a family $\cG\subset \R^{\cZ}$ of real-valued functions on   a probability space $(\cZ, \rho)$, the empirical Rademacher complexity of $\cG$ with respect to a sample $\cS = \{\z_1, \cdots \z_n\}\sim \rho^n$ is defined as \citep{Mohri:2012}: 
\beq \label{def:empRad}
\widehat{\mathcal{R}}_\cS(\cG) = \E_{{\bm \sigma}\in \{\pm 1\}^n} \left[\sup_{g\in \cG} \frac{1}{n}  \sum_{i=1}^n \sigma_i g(\z_i)\right],
\eeq  
where the expectation is over $n$  i.i.d uniform random variables $\sigma_1, \cdots \sigma_n \in \{\pm 1\} $. For any $n \geq 1$, the Rademacher complexity with respect to samples of size $n$ is then 
$\mathcal{R}_n(\cG) = \E_{\cS\sim\rho^{n}} \widehat{\mathcal{R}}_\cS(\cG)$. 

\subsection{Generalization Bounds} 
\label{sec:margin_bounds}

Generalization bounds based on Rademacher complexity are standard \citep{Bartlett:2017, Mohri:2012}. We give here one instance of such a bound, relevant for classification task. 

{\bf Setup.} We consider a family $\cF$ of functions  $f_\w \maps \cX  \to \R^c$ that output a score or probability $f_\w(\x)[y]$ for each class $y \in  \{1\cdots c\}$ (we take $c=1$ for binary classification).  The task is to find a predictor $f_\w \in \cF$ with small expected classification error, which can be expressed e.g. as 
\beq \label{eq:classif_error}
L_0(f_\w)  \!= \!\mathbb{P}_{(\x, y)\sim \rho} \left\{\mu(f_\w(\x), y)<0\right\}
\eeq 
where $\mu(f(\x), y)$ denotes the {\bf margin},
\beq \label{eq:margin}
\mu(f(\x), y) = \begin{cases} f(\x) y \quad &\mbox{binary case} \\ 
f(\x)[y] - \max_{y'\not=y}f(\x)[y'] \quad &\mbox{multiclass case}
\end{cases}
\eeq

{\bf Margin Bound.}  We consider the  {\bf margin loss},
\beq \label{eq:marginloss}
 \ell_\gamma(f_\w(\x), y)) = \phi_\gamma(\mu(f_\w(\x), y))
\eeq  
where $\gamma>0$, and $\phi_\gamma$ is the {\bf ramp} function:  $\phi_\gamma(u)= 1$ if $u\leq 0$, $\phi(u)=0$ if $u>\gamma$ and $\phi(u) = 1-u/\gamma$ otherwise. We have the following bound for the expected error (\ref{eq:classif_error}). 
With probability at least $1-\delta$ over the draw $\cS = \{\z_i = (\x_i, y_i)\}_{i=1}^n$ of size $n$, the following holds for all $f_\w \in \cF$ 
\citep[Theorems 4.4.and 8.1]{Mohri:2012}:
\beq \label{eq:genbound}
L_0(f_\w) \leq \widehat{L}_{\gamma}(f_\w) + 2 \widehat{\cR}_\cS(\ell_\gamma(\cF, \cdot)) + 3\sqrt{\frac{\log \frac{2}{\delta}}{2n}}
\eeq
where $\widehat{L}_{\gamma}(f_\w) = \frac{1}{n} \sum_{i=1}^n \ell_\gamma(f_\w(\x_i), y_i)$ is the empirical margin error and $\ell_\gamma(\cF, \cdot)$ is the {\bf loss class},
\beq \label{eq:marginlosscl}
\ell_\gamma(\cF, \cdot)=\{(\x, y) \mapsto \ell_\gamma(f_\w(\x), y) \, | \, f_\w \in \cF\} 
\eeq 
For binary classifiers, 
because $\phi_{\gamma}$ is $1/\gamma$-Lipschitz, 
we have in addition \beq \cR_\cS(\ell_\gamma(\cF, \cdot)) \leq \frac{1}{\gamma} \cR_\cS(\cF)
\eeq
 by   Talagrand's contraction lemma \citep{LedouxTalagrandbook}  \citep[see e.g.][lemma 4.2 for a detailed proof]{Mohri:2012}.

\subsection{Complexity Bounds: Proofs}
\label{sec:proofthm1}
We first derive standard bounds for the linear classes of scalar functions,
\beq 
\cF^A_{M_{\!A}} = \{f_\w \maps \x \mapsto \langle \w, \Phi(\x) \rangle \,\, | \,\, \|\w \|_{\! A} \leq M_{\!A}\}
\eeq 
\begin{prop} \label{theo:Abounds}
The empirical Rademacher complexity of $\cF^A_{M_{\!A}}$ is bounded as,
\beq \label{appeq:Abounds}
\widehat{\mathcal{R}}_\cS(\cF^A_{M_{\!A}}) \leq  (M_{\!A}/n) \sqrt{\Tr {\bm K}_{\! A}} 
\eeq
where $({\bm K}_{\! A})_{ij} = k_{\!A}(\x_i, \x_j)$ is the kernel matrix associated to the rescaled features $A^{-1} \Phi$. 
\end{prop}
\begin{proof} 
We use the notation of Section \ref{sec:complexity_measure}. For given Rademacher variables ${\bm \sigma} \in \{\pm 1 \}^n$, we have,
\beqa  \label{step1}
\sup_{f\in \cF^A_{\!M_{\!A}}} \, \sum_{i=1}^n  \sigma_i f(\x_i) 
&=& \sup_{\|\w \|_{\! A} \leq M_{\!A}} \sum_{i=1}^n \sigma_i \langle \w, \Phi(\x_i) \rangle \nonumber \\
&=& \sup_{\|A^{\!\top} \w \|_2 \leq M_{\!A}} \sum_{i=1}^n \sigma_i \langle A^\top \w, A^{-1}\Phi(\x_i) \rangle  \nonumber \\
&=& \sup_{\|\tilde{\w} \|_2 \leq M_{\!A}} \langle \tilde{\w}, \sum_{i=1}^n \sigma_i A^{-1} \Phi(\x_i) \rangle \nonumber \\
&=& M_{\!A} \left\|\sum_{i=1}^n \sigma_i A^{-1} \Phi(\x_i)\right\|_2  \nonumber \\ &=& M_{\!A} \sqrt{{\bm \sigma}^{\!\top} \K_{\!A} {\bm \sigma}}
\eeqa
From (\ref{step1}) and the definition (\ref{def:empRad})  we obtain:
\beqa \label{step2}
\widehat{\mathcal{R}}_\cS(\cF^A_{M_{\!A}}) &=& \frac{M_{\!A}}{n}\E_{\bm \sigma} \left[\sqrt{{\bm \sigma}^{\!\top} \K_{\!A} {\bm \sigma}}\right] \nonumber \\ 
&\leq& \frac{M_{\!A}}{n} \sqrt{\E_{\bm \sigma} \left[{\bm \sigma}^{\!\top} \K_{\!A} {\bm \sigma} \right]} \nonumber \\
&\leq& \frac{M_{\!A}}{n} \sqrt{\Tr \K_{\!A}} 
\eeqa 
where we used Jensen's inequality to pass $\E_\sigma$ under the root, and that $\E[\sigma_i] = 0$ and $\sigma_i^2 = 1$ for all $i$. 
\end{proof} 
We now extend the result to  the families (\ref{eq:dynAclasses}) of learning flows: 
\beq 
\cF^{\! \bm A}_{\bm m} = \{f_\w \maps \x \mapsto {\textstyle \sum_t} \langle \delta \w_t, \Phi(\x) \rangle  \,\, | \, \, \|\delta\w_t\|_{\!A_t} \leq m_t\} 
\eeq

\begin{theo}[Theorem \ref{theo:comp_lf} restated]
The empirical Rademacher complexity of $\cF^{\! \bm A}_{\bm m}$  is bounded as,
\beq \label{appeq:Abound}
\widehat{\mathcal{R}}_\cS(\cF^{\!\bm A}_{\bm m})  \leq {\textstyle \sum_t} (m_t/n) \sqrt{\Tr {\bm K}_{\!A_t}} 
\eeq  
where $(\K_{\!A_t})_{ij} = k_{\!A_t}(\x_i, \x_j)$ is the kernel matrix associated to the rescaled features $A^{-1}_t \Phi$. 
\end{theo}
\begin{proof}
This is simple extension of the previous proof:
\beqa 
\sup_{f\in \cF^{\!\bm A}_{\! \bm m}} \, \sum_{i=1}^n  \sigma_i f(\x_i) 
&=& \sup_{\|\delta \w_t\|_{\!A_t} \leq m_t }  \sum_{i=1}^n \sigma_i \sum_t \langle \delta\w_t, \Phi(\x_i)\rangle \nonumber \\
&=&  \sum_t \sup_{\| \tilde{\delta} \w_t  \|_2 \leq m_t } \langle \tilde{\delta} \w_t, \sum_{i=1}^n \sigma_i A_t^{-1} \Phi(\x_i) \rangle \nonumber \\
&=& \sum_t m_t \sqrt{{\bm \sigma}^{\!\top} \K_{\!A_t} {\bm \sigma}}
\eeqa 
and we conclude as in (\ref{step2}).  
\end{proof}

Finally, we note that the same result can be formulated in terms of an evolving feature map $\Phi_t = A_t^{-1} \Phi$ with kernel 
$k_t(\x, \tilde{\x}) =\langle 
\Phi_t(\x), \Phi_t(\tilde{\x})\rangle$
In fact by reparametrization invariance,  the function updates can also be written as $\delta f_{\w_t}(\x) = \langle \tilde{\delta} \w_t, \Phi_t(\x) \rangle$ where $\tilde{\delta} \w_t = A_t^{\!\top} \delta \w_t$. The function class (\ref{eq:dynAclasses}) can equivalently be written as $\cF^{\! \bm A}_{\bm m} = \cF^{\! \bm \Phi}_{\bm m}$ where ${\bm \Phi}$ denotes a fixed sequence  of feature maps, ${\bm \Phi}= \{\Phi_t\}_t$ and
\beq \label{eq:dynPhiclasses}
\cF^{\! \bm \Phi}_{\bm m} = \{f_\w \maps \x \mapsto {\textstyle \sum_t} \langle \tilde{\delta} \w_t, \Phi_t(\x) \rangle  \, | \,  \|\tilde{\delta}\w_t\|_{2} \leq m_t\}
\eeq

In this formulation, the result (\ref{appeq:Abound}) is expressed as, 
\beq \label{app:kern_newbound} 
\widehat{\mathcal{R}}_\cS(\cF^{\!\bm \Phi}_{\bm m})  
\leq {\textstyle \sum_t} (m_t/n) \sqrt{\Tr {\bm K}_{t}} \eeq  
where $({\bm K}_{t})_{ij}=k_t(\x_i, \tilde{\x_j})$ is the kernel matrix associated to the feature map $\Phi_t$.

\subsection{Bounds for Multiclass Classification}
\label{app:bounds_multiclass}

The generalization bound (\ref{eq:genbound}) is based on the {\bf margin loss class} (\ref{eq:marginlosscl}).
In this section, we show how to bound  $\widehat{\cR}_\cS(\ell_\gamma(\cF, \cdot))$ in terms of tangent kernels for the original class $\cF$ of functions $f_\w \maps \cX \to \R^c$ instead. Although the proof is adapted from standard techniques, to our knowledge  Lemma \ref{app:cool_lemma}  and Theorem \ref{theo:Abounds} below are new results.  In what follows,  we denote by $\mu_\cF$ the margin class,
\beq 
\mu_\cF = \{ (\x, y) \to \mu(f_\w(\x), y) \, | \, f_\w \in \cF \}
\eeq 
where $\mu(f_\w(\x), y))$ is the margin (\ref{eq:margin}). We also define,  for each $y \in \{1\cdots c\}$,
\beq 
\cF_y = \{\x \mapsto f_\w(\x)[y] \,| \, f_\w \in \cF\}, \quad \mu_{\cF, y} = \{\x \mapsto \mu(f_\w(\x), y) \,| \, f_\w \in \cF \}  
\eeq 
\begin{lem} 
\label{app:cool_lemma}
The following inequality holds: 
\beq \label{lemma:comp_lossbound}
\widehat{\cR}_\cS(\ell_\gamma(\cF, \cdot)) \leq \frac{c}{\gamma} \sum_{y=1}^c \widehat{\cR}_\cS(\cF_y)
\eeq
\end{lem} 
\begin{proof}
We first follow the first steps of the proof of \citep[][Theorem 8.1]{Mohri:2012} to show that  
\beq \label{proof:step1}
\widehat{\cR}_\cS(\ell_\gamma(\cF, \cdot)) \leq  \frac{1}{\gamma} \sum_{y=1}^c \widehat{\cR}_\cS(\mu_{\cF, y})
\eeq 
We reproduce these steps here for completeness: first, it follows from the $1/\gamma$-Lipschitzness of the ramp loss $\phi_\gamma$ in (\ref{eq:marginloss}) and Talagrand's contraction lemma \citep[lemma 4.2]{Mohri:2012} that  
\beq \label{proof:step2}
\widehat{\cR}_\cS(\ell_\gamma(\cF, \cdot)) \leq \frac{1}{\gamma}\widehat{\cR}_\cS(\mu_\cF)
\eeq 
Next, we write
\beqa 
\widehat{\cR}_\cS(\mu_\cF) &:=& \frac{1}{n} \E_{{\bm \sigma}}\left[ \sup_{f_\w \in \cF} \sum_{i=1}^n \sigma_i \mu(f_\w(\x_i), y_i) \right]\nonumber \\
&=& 
\frac{1}{n} \E_{{\bm \sigma}}\left[ \sup_{f_\w \in \cF} \sum_{i=1}^n \sigma_i \sum_{y=1}^c \mu(f_\w(\x_i), y) \, \delta_{y, y_i}\right]\nonumber \\
&=& 
\frac{1}{n} \sum_{y=1}^c \E_{{\bm \sigma}} \left[\sup_{f_\w \in \cF} \sum_{i=1}^n \sigma_i \mu(f_\w(\x_i), y) \, \delta_{y, y_i}\right]
\eeqa
where $\delta_{y, y_i} = 1$ if $y=y_i$ and $0$ otherwise; the second  inequality follows from the sub-additivity of $\sup$.  Substituting $\delta_{y, y_i} = \frac12(\epsilon_i + \frac12)$ where  $\epsilon_i = 2\delta_{y, y_i} -1 \in \{\pm 1\}$, 
we obtain
\beqa
\widehat{\cR}_\cS(\mu_\cF) &\leq&  \frac{1}{2n} \sum_{y=1}^c \E_{{\bm \sigma}} \left[\sup_{f_\w \in \cF} \sum_{i=1}^n (\epsilon_i \sigma_i)  \mu(f_\w(\x_i), y) \right] + \frac{1}{2n} \sum_{y=1}^c \E_{{\bm \sigma}} \left[\sup_{f_\w \in \cF} \sum_{i=1}^n  \sigma_i  \mu(f_\w(\x_i), y) \right] \nonumber \\
&=& \sum_{y=1}^c \frac{1}{n}  \E_{{\bm \sigma}} \left[\sup_{f_\w \in \cF} \sum_{i=1}^n \sigma_i  \mu(f_\w(\x_i), y) \right] \nonumber \\ 
&=& \sum_{y=1}^c \widehat{\cR}_\cS(\mu_{\cF, y})
\eeqa
Together with (\ref{proof:step2}), this leads to (\ref{proof:step1}).

Now, spelling out $\mu(f_\w(\x_i, y))$ gives
\beqa
\widehat{\cR}_\cS(\mu_{\cF, y}) &=& 
\frac{1}{n}  \E_{{\bm \sigma}} \left[\sup_{f_\w \in \cF} \sum_{i=1}^n \sigma_i  (f_\w(\x_i)[y] - \max_{y' \not=y} f_\w(\x_i)[y']) \right] \nonumber \\
&=&  \widehat{\cR}_\cS(\cF_y) + 
\frac{1}{n} \E_{{\bm \sigma}} \left[\sup_{f_\w \in \cF} \sum_{i=1}^n (-\sigma_i)  \max_{y' \not=y} f_\w(\x_i)[y'] \right] \nonumber \\
&=&\widehat{\cR}_\cS(\cF_y) + 
\frac{1}{n} \E_{{\bm \sigma}} \left[\sup_{f_\w \in \cF} \sum_{i=1}^n \sigma_i  \max_{y' \not=y} f_\w(\x_i)[y'] \right] \nonumber \\
&\leq& 
\widehat{\cR}_\cS(\cF_y) + \widehat{\cR}_\cS(\cG_y)  
\eeqa
where 
$\cG_y = \left\{\max\{f_{y'} : y' \not= y\} \,|\, f_{y'} \in \cF_{y'}\right\}$. Now  \citep[lemma 8.1]{Mohri:2012}  show that  $\widehat{\cR}_\cS(\cG_y) \leq \sum_{y' \not= y} \widehat{\cR}_\cS(\cF_{y'})$. This leads to
\beqa
\sum_{y=1}^c \widehat{\cR}_\cS(\mu_{\cF, y}) 
&\leq& \sum_{y=1}^c \widehat{\cR}_\cS(\cF_y) + \sum_{y=1}^c \sum_{\substack{y'=1 \\ y'\not=y}}^c \widehat{\cR}_\cS(\cF_{y'}) \nonumber \\
&=& 
\sum_{y=1}^c \widehat{\cR}_\cS(\cF_y) + (c-1) \sum_{y=1}^c\widehat{\cR}_\cS(\cF_y) \nonumber \\
&=& c \sum_{y=1}^c \widehat{\cR}_\cS(\cF_y)
\eeqa 
Substituting in (\ref{proof:step1}) finishes the proof.
\end{proof}
In the linear case, this results leads to  analogous theorems as in  \ref{sec:proofthm1} in the multiclass setting. For example, 
considering the linear families of functions $\cX \to \R^c$,
\beq 
\cF^A_{M_{\!A}} = \{\x \mapsto f_\w(\x)[y] :=  \langle \w, \Phi(\x)[y] \rangle \,\, | \,\, \|\w \|_{\! A} \leq M_{\!A}\}
\eeq 
where $(\x, y) \mapsto \Phi(\x)[y]$ is some joint feature map, we have the following
\begin{theo} 
\label{theo:Abounds}
The emp. Rademacher complexity of the margin loss class $\ell_\gamma(\cF^A_{M_{\!A}}, \cdot)$ is bounded as,
\beq \label{eq:Abounds}
\widehat{\mathcal{R}}_\cS(\ell_\gamma(\cF^A_{M_{\!A}}, \cdot)) \leq  (c^{3/2} M_{\!A}/\gamma n) \sqrt{\Tr {\bm K}_{\! A}} 
\eeq
where $({\bm K}_{\! A})^{yy'}_{ij}$ is the kernel $nc \times nc$ matrix associated to the rescaled features $A^{-1} \Phi(\x)[y]$. 
\end{theo}
\begin{proof} 
Eq.\ref{lemma:comp_lossbound}, and Theorem \ref{theo:Abounds} applied to each linear family $\cF_y$ of (scalar) functions leads to 
\beq 
\widehat{\mathcal{R}}_\cS(\ell_\gamma(\cF^A_{M_{\!A}}, \cdot)) \leq \frac{c}{\gamma} \sum_{y=1}^c \frac{M_{\!A}}{n} \sqrt{\Tr \K^{yy}_A}
\eeq 
where $\Tr \K^{yy}_A := \sum_{i=1}^n (\K_A)^{yy}_{ii}$ is computed w.r.t to the indices $i=1,...,n$ for fixed $y$. Passing the average  $\frac{1}{c} \sum_{y=1}^c$ under the root using Jensen inequality, we conclude: 
\beqa 
\widehat{\mathcal{R}}_\cS(\ell_\gamma(\cF^A_{M_{\!A}}, \cdot)) &\leq& \frac{c^2 M_{\!A}}{\gamma n} \sqrt{\frac{1}{c} \sum_{y=1}^c \Tr \K^{yy}_A} \nonumber \\
&=& \frac{c^{3/2} M_{\!A}}{\gamma n} \sqrt{\Tr \K_{\!A}}
\eeqa 
\end{proof}
The proof of the extension of these bounds to families learning flows follows the same line as in \ref{sec:proofthm1}.

\subsection{Which Norm for Measuring Capacity?}
\label{appendix:norm_capacity}

Implicit biases of gradient descent are relatively well understood in linear models (e.g \citet{ImplicitBias2018}). For example when using square loss, it is  well-known that gradient descent (initialized in the span of the data) converges to minimum $\ell^2$ norm (resp. RKHS norm) solutions in parameter space (resp. function space). 
Yet, as pointed out by \citet{Belkin2018, Muthukumar:2020}, measuring capacity in terms of such norms is not coherently linked with generalization in practice. Here we discuss this issue 
by highlighting the critical dependence  of meaningful norm-based capacity on the geometry defined by the features. We use the notation of Section \ref{sec:lin}: ${\bm \Phi}= \sum_{j=1}^n \sqrt{\lambda_j} \u_j \v_j^\top$  denote the $n\times P$ feature matrix and its SVD decomposition.

A standard approach is to measure capacity in terms of the $\ell^2$ norm the weight vector, e.g using bounds (\protect\ref{eq:Abound}) with $A=\mathrm{Id}$.
If the distribution of solutions $\w^\ast_\cS$, where $\cS \sim \rho^n$ is sampled from the input distribution, is reasonably isotropic, taking the smallest $\ell^2$ ball  containing them (with high probability) gives an accurate description of the class of  trained models.  However for very anisotropic distributions, the solutions do not fill any such ball so describing trained models in terms of $\ell^2$ balls is wasteful \citep{Scholkopf:1999}. 
 
Now, for minimum $\ell^2$ norm interpolators \citep{hastie_09},
\beq \label{eq:minnorm}
\w^\ast\! =\! {\bm \Phi}^{\!\top} \K^{-1} \y = \sum_{j=1}^n \frac{\u_j^{\!\top} \y}{\sqrt{\lambda_j}}\, \v_j,
\eeq
where $\K = {\bm \Phi} {\bm \Phi}^\top$ is the kernel matrix, the solution distribution  typically inherits  the anisotropy of the features. For example, if  $y_i = \bar{y}(\x_i) + \varepsilon_i$ where $\varepsilon_i \sim \cN(0, \sigma^2)$, 
the covariance of the solutions with respect to noise is $\mbox{cov}_{\bm \varepsilon}[\w^\ast, \w^\ast] = \sum_j \frac{\sigma^2}{\lambda_j} {\bm v_j} {\bm v}^{\!\top}_j$, which scales as $1/\lambda_j$ along ${\v}_j$.  

To visualize this on a simple setting, we consider $P$ random features of a RBF kernel\footnote{We used \href{https://github.com/scikit-learn/scikit-learn/blob/0fb307bf3/sklearn/kernel_approximation.py}{RBFsampler} of scikit-learn, which implements a variant of Random Kitchen Sinks \citep{Rahimi:2007} to approximate the feature map of a RBF kernel with parameter $\gamma=1$.},
fit on 1D data $\x$ modelled by $N$ equally spaced  points in  $[-a, a]$. In this setting, the (true) feature map is represented by a $N\times P$ matrix with SVD $\Phi = \sum_{j} \sqrt{l_j} {\bm \psi}_j {\bm \varphi}^{\!\top}_j $. We assume the (true) labels are defined by the deterministic function $y(\x) = \mbox{sign}({\bm \psi_1}(\x))$.  To highlight the effect of feature anisotropy, we further rescale the singular values as $l^c_j =  1 + c(l_j-1)$ so as to interpolate between whitened features $(c\!=\!0)$ and the original ones $(c\!=\! 1)$. We set $P\!=\!N\!=\!1000$.
\begin{figure*}[t]
\begin{subfigure}[t]{0.45\linewidth}
\centering
\includegraphics[width=\linewidth]{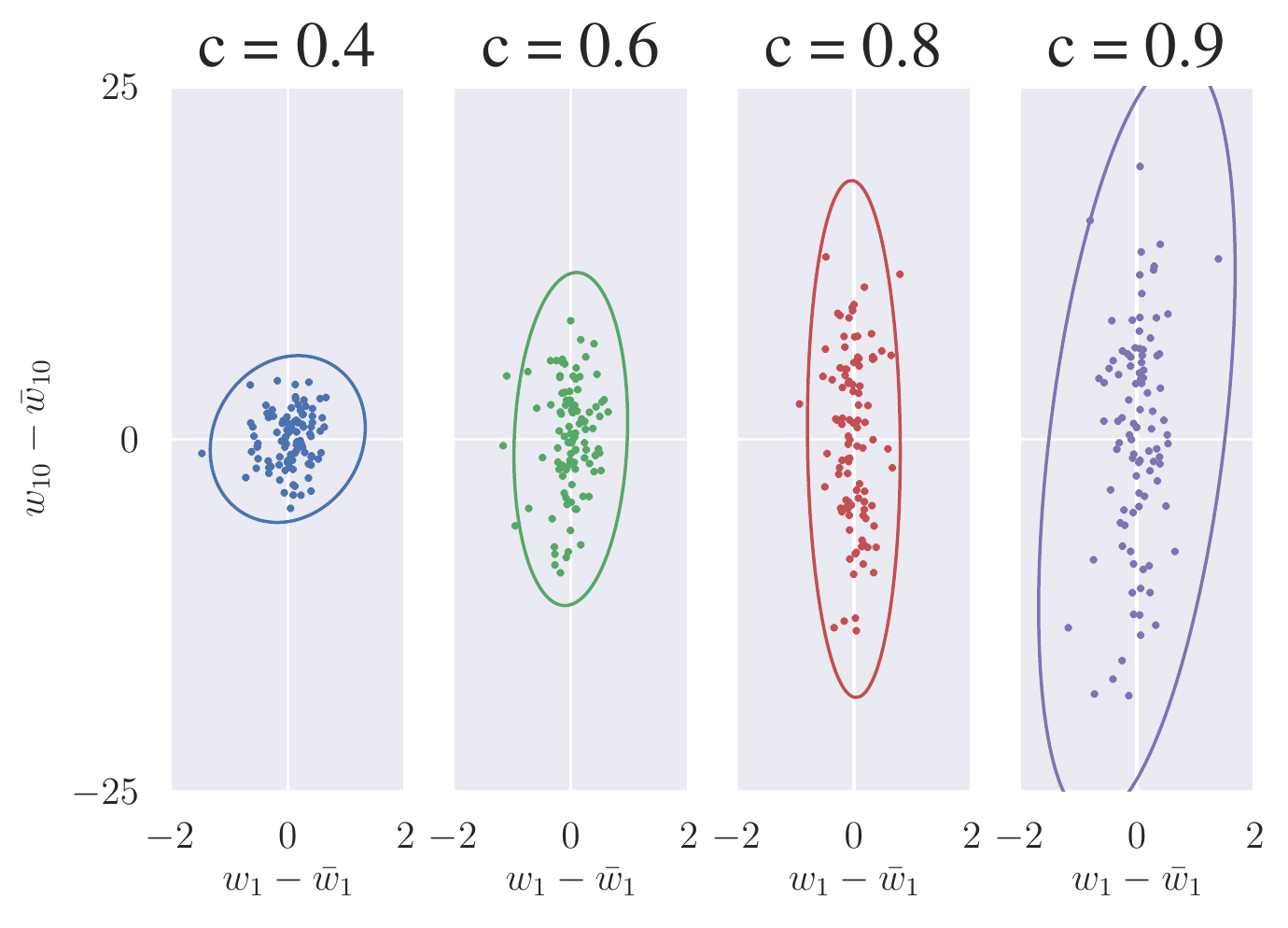}
\end{subfigure}
\begin{subfigure}[t]{0.54\linewidth}
\centering
\includegraphics[width=\linewidth]{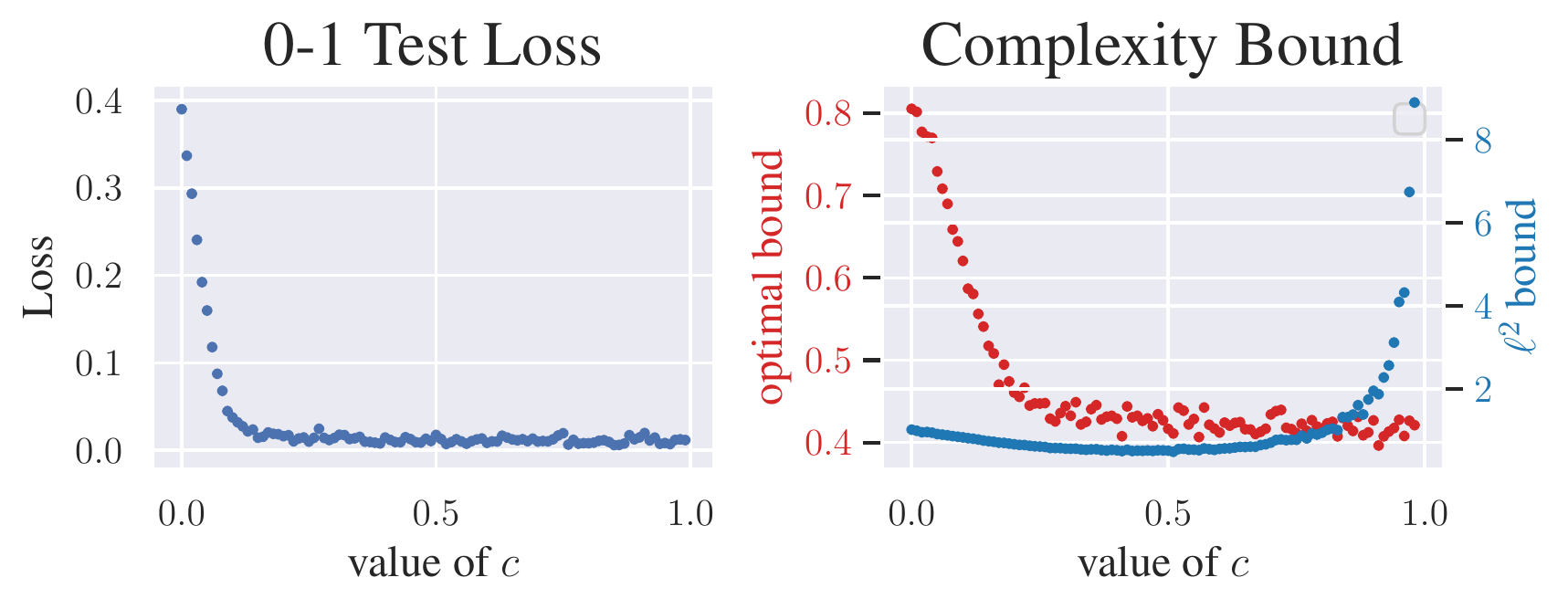}
\end{subfigure}

\caption{\small {\bf Left:} 2D  projection of the minimum $\ell^2$-norm interpolators  $\w^\ast_{\!\cS}, \, 
\cS \sim \rho^n$,  
for linear models $f_\w = \langle \w, \Phi_c\rangle$, as the feature scaling factor varies from $0$ (white features) to $1$ (original, anisotropic features).   
    For larger $c$, the solutions scatter in a very anisotropic way.  {\bf Right:} Average test classification loss and complexity bounds (\ref{eq:Abounds}) with $A=\mathrm{Id}$ (blue plot) for the solution vectors $\w^\ast_{\!\cS}$, as we increase the scaling factor $c$. As feature anisotropy increases, the bound becomes increasingly loose and fails to reflect the shape of the test error. By contrast, the bound  (\protect\ref{eq:Abound}) with $A$ optimized as in Proposition \protect\ref{prop:optimal_norm} (red plot) does not suffer from this problem.}
\label{fig:interpolated}
\end{figure*}
Fig~\ref{fig:interpolated} (left) shows 2D projections in the plane $({\bm \varphi}_1, {\bm \varphi}_{10})$ of the (centered) minimum $\ell_2$   norm solutions $\w^\ast_{\!\cS}-\E_{\cS}\w^\ast_{\!\cS}$, for a pool of 100 training (sub)samples $\cS$ of size $n=50$, for increasing values of the scaling factor $c$. As $c$ approaches $1$, the solutions begin to scatter in a very anisotropic way in parameter space; as shown in Fig~\ref{fig:interpolated} (right), the complexity  bound (\ref{appeq:Abounds}) based on the $\ell_2$ norm, i.e $A=\mathrm{Id}$ (blue plot), becomes increasingly loose and fails to reflect the shape of the test error.

To find a more meaningful capacity measure, Prop \ref{theo:Abounds} suggests optimizing the bound (\ref{eq:Abound}) with $M_{\!A} = \|\w^\ast \|_{\! A}$, over a given class of rescaling matrices $A$. We give an example of this in the following Proposition.
\begin{prop} \label{prop:optimal_norm}
Consider the class of  matrices $A_\nu = \sum_{j=1}^n \sqrt{\nu_j} \v_j \v_j^\top + \mathrm{Id}_{\mathrm{span}\{\v\}^\perp}$, which act as mere rescaling of the singular values of the feature matrix. Any minimizer of the upper bound (\ref{appeq:Abounds}) for the mininum $\ell^2$-norm interpolator takes the form  
\beq  \label{eq:nusol2} 
\nu_{j}^\ast = \kappa \frac{\sqrt{\lambda_j}}{|\v_j^{\!\top} \w^\ast|} = \kappa \frac{\lambda_j}{|\u_j^{\!\top} \y|}
\eeq
where $\kappa>0$ is a constant independent of $j$. 
\end{prop} 
\begin{proof} 
From (\ref{eq:minnorm}) and the definition of  $A_\nu$, we first write 
\beq 
\|\w^\ast \|^2_{A_\nu} = \sum_{j=1}^n \frac{\nu_j}{\lambda_j} (\u_j^{\!\top} \y)^2, \quad 
\Tr \K_{\!A_\nu} = \sum_{j=1}^n \frac{\lambda_j}{\nu_j}
\eeq 
The product of the above two terms has the critical points $\nu_j^\ast$, $j=1\cdots n$ which satisfy
\beq 
\frac{(\u_j^{\!\top} \y)^2}{\lambda_j} \Tr \K_{\!A_\nu} - \frac{\lambda_j}{\nu^{\ast 2}_j} \|\w^\ast \|^2_{A_\nu} = 0 
\eeq
giving the desired result $\nu^\ast_j \propto \lambda_j / |\u_j^{\!\top} \y|$.
\end{proof} 
In the context of Proposition \ref{prop:optimal_norm}, we see that the optimal norm $\|\cdot \|_{\!A_{\nu^\ast}}$ depends both on the feature geometry -- through the singular values --  and on the task -- through the labels --. As  shown in Fig 1 (right, red plot), in the above RBF feature setting,  the resulting optimal bound on the Radecher complexity has a much nicer behaviour than the standard bound  based on the $\ell^2$ norm.\footnote{Note however that, since the optimal norm depends on the sample set $\cS$, the resulting complexity bound does not directly yield a high probability bound on the generalization error as in (\ref{eq:genbound}). The more thorough analysis, which requires promoting (\ref{eq:genbound}) to uniform bounds over the choice of matrix $A$, is left for future work.} 



 

\subsection{SuperNat: Proof of Prop \ref{prop:optimal_supernorm}}

Prop.~\ref{prop:optimal_supernorm} is a {\it local} version of Prop \ref{prop:optimal_norm}, where  the feature rescaling factors are applied at each step of the training algorithm. The procedure is described in Fig \ref{Fig:SuperNat} (left); the term to be optimized   shows up in Step 2. With the chosen class of matrices described in Prop \ref{prop:optimal_supernorm}, the action ${\bm \Phi}_t \to A_{\nu}^{-1} {\bm \Phi}_{t}$ merely rescale its singular values $\lambda_{jt} \to \lambda_{jt} / \nu_j$, leaving  its singular vectors $\u_j, \v_j$ unchanged.
\begin{prop}[Prop \ref{prop:optimal_supernorm} restated]
For the class of rescaling matrices $A_\nu$ defined in Prop \ref{prop:optimal_norm}, any minimizer in Step 2 in Fig \ref{Fig:SuperNat}, where $\delta {\w_{\!\mbox{\tiny GD}}} = - \eta \nabla_{\!\w} L$, takes the form  
\beq  \label{eq:nusol2} 
\nu_{jt}^\ast = \kappa \frac{1}{|\u_j^{\!\top} \nabla_{\f_\w} L|}
\eeq
where $\kappa>0$ is a constant independent of $j$. 
\end{prop}
\begin{proof} 
Using the chain rule and the SVD of the feature map $\Phi_t$ we write the gradient descent  updates at iteration $t$ of SuperNat as
\begin{align} 
\delta {\w_{\!\mbox{\tiny GD}}} &= - \eta {\bm \Phi}_{\!t}^\top \nabla_{\!\f_\w} L \\ 
&= -\eta \sum_{j=1}^n \sqrt{\lambda_{j t}} (\u_{j}^{\!\top} \nabla_{\!\f_\w} L) \, \v_{j},
\end{align}
From the definition of $A_\nu$, we then
spell out 
\beq 
\|\delta {\w_{\!\mbox{\tiny GD}}}\|^2_{A_\nu} = \eta^2  \sum_{j=1}^n (\nu_j \lambda_j) (\u_{j}^{\!\top} \nabla_{\!\f_\w} L)^2, \quad 
\|A_{\nu}^{-1} {\bm \Phi}_{t}\|_F:=\Tr \K_{\!t A_\nu} = \sum_{j=1}^n \frac{\lambda_j}{\nu_j}
\eeq 
The product of the above two terms has the critical points $\nu_j^\ast$, $j=1\cdots n$ which satisfy
\beq 
\lambda_j (\u_j^{\!\top} \nabla_{\!\f_\w} L)^2 \Tr \K_{\!A_\nu} - \frac{\lambda_j}{\nu^{\ast 2}_j} \|\delta {\w_{\!\mbox{\tiny GD}}} \|^2_{A_\nu} = 0 
\eeq
giving the desired result $\nu_j^\ast \propto 1 / |\u_j^{\!\top} \nabla_{\f_\w} L|$.
\end{proof}

\section{Additional experiments}
\label{appendix:exps}

\subsection{Synthetic Experiment: Fig.~\ref{fig:NTK_Disk}}
\label{appendix:disk_principal_components}

To visualize the adaptation of the tangent kernel to the task during training, we perform the following synthetic experiment. We train a 6-layer deep 256-unit wide MLP on $n=500$ points of the $\mathrm{Disc}$  dataset $(\x, y)$ where $\x \sim \mbox{Unif}[-1,1]^2$ and $y(\x) = \pm 1$ depending on whether is within the disk of center 0 and radius $\sqrt{2/\pi}$, see Fig \ref{fig:Disk}.  Fig. \ref{fig:NTK_Disk} in the main text shows visualizations of  eigenfunctions sampled using a grid of $N=2500$ points on the square, and ranked in non-increasing order of the spectrum $\lambda_{1} \geq \cdots \geq \lambda_{N}$. 
\begin{figure}
	\centering
\begin{center}
\includegraphics[width=0.3\textwidth]{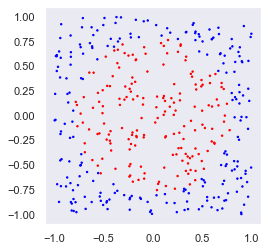}
\hspace{1cm}
\includegraphics[width=0.3\textwidth]{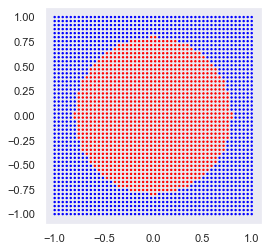}
	\caption{\label{fig:Disk} 
	$\mathrm{Disk}$ dataset. {\bf Left}: Training set of $n=500$ points $(\x_i, y_i)$ where $\x \sim \mbox{Unif}[-1,1]^2$, $y_i = 1$ if $\|x_i\|_2 \leq r = \sqrt{2/\pi}$ and $-1$ otherwise. {\bf Right}: Large test sample (2500 points forming a $50\times 50$ grid) used to evaluate the tangent kernel.}
\end{center}
\end{figure}
After a number of iterations, we begin to see the class structure (e.g. boundary circle) emerge in  the top eigenfunctions. We note also an increasingly fast spectrum decay (e.g $\lambda_{20} / \lambda_{1} = 1.5\%$ at iteration $0$ and $0.2\%$ at iteration $2000$). The interpretation is that the kernel stretches in directions of high correlation with the labels.

\subsection{More Alignment Plots}
\label{appendix:uncentered_kernel}

{\bf Varying datasets and architectures:} Fig \ref{fig:cka_alignment_additional}.

{\bf Uncentered kernel Experiments:} Fig \ref{fig:uncentered_alignment}.
The evolution of the alignment to the \emph{uncentered} kernel, in order to assess whether this effect is consistent when removing centering. The experimental details are the  same as in the main text; we also observe a similar increase of the alignment as training progresses.



\begin{figure*}
\begin{subfigure}[t]{0.33\linewidth}
\centering
\includegraphics[width=\linewidth]{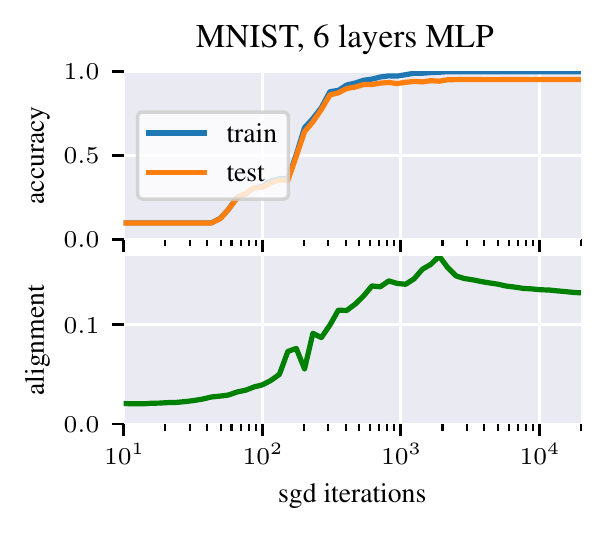}
\end{subfigure}
\begin{subfigure}[t]{0.33\linewidth}
\centering
\includegraphics[width=\linewidth]{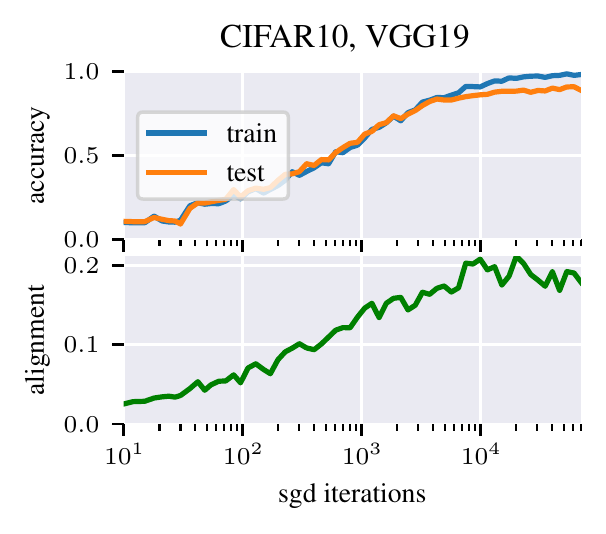}
\end{subfigure}
\begin{subfigure}[t]{0.33\linewidth}
\centering
\includegraphics[width=\linewidth]{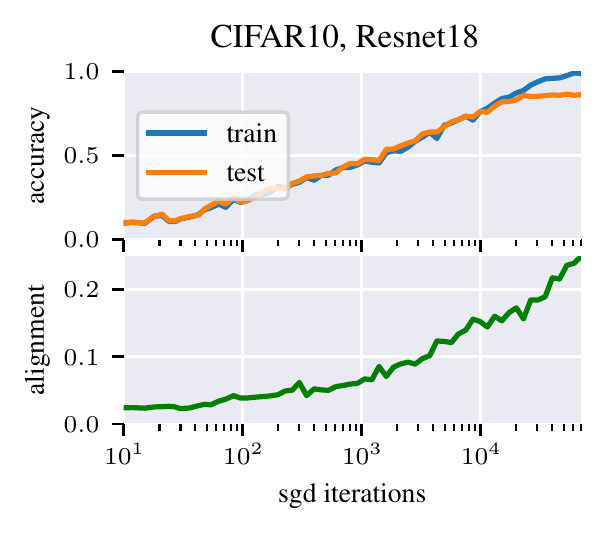}
\end{subfigure}

\caption{\small Evolution of the CKA between the tangent kernel and the class label kernel $K_Y = YY^T$ measured on a held-out test set for different architectures: \textbf{(left)} 6 layers of 80 hidden units MLP on MNIST \textbf{(middle)} VGG19 on CIFAR10 \textbf{(right)} Resnet18 on CIFAR10. We observe an increase of the alignment to the target function.}\label{fig:cka_alignment_additional}
\end{figure*}

\begin{figure*}
\begin{subfigure}[t]{0.33\linewidth}
\centering
\includegraphics[width=\linewidth]{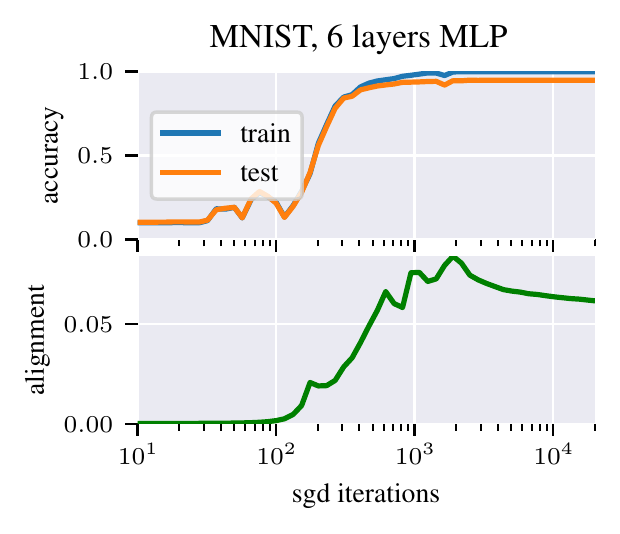}
\end{subfigure}
\begin{subfigure}[t]{0.33\linewidth}
\centering
\includegraphics[width=\linewidth]{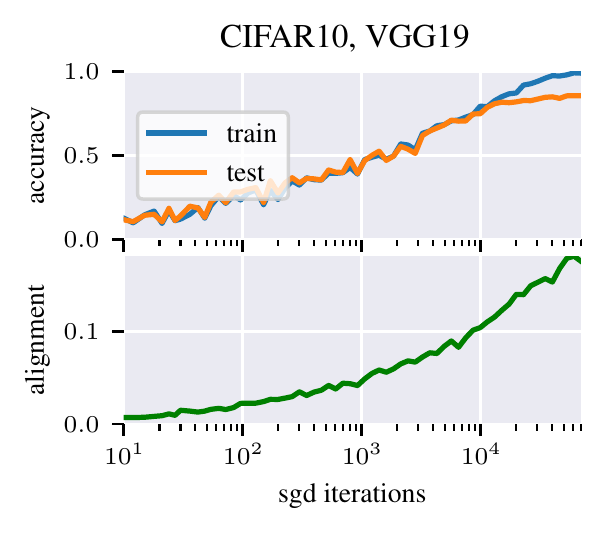}
\end{subfigure}
\begin{subfigure}[t]{0.33\linewidth}
\centering
\includegraphics[width=\linewidth]{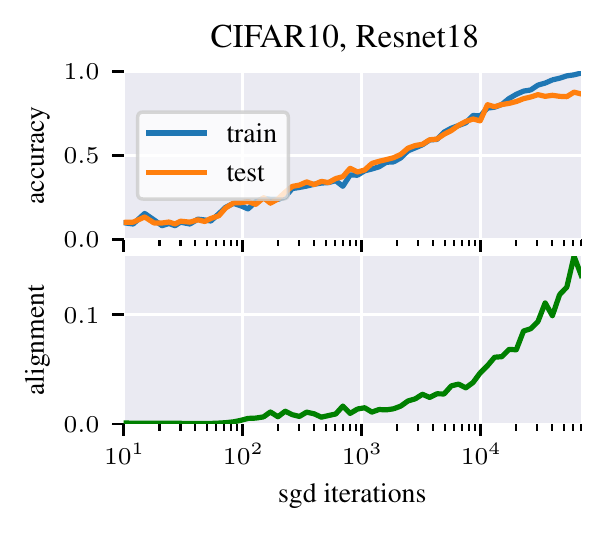}
\end{subfigure}
\caption{\small Same as figure \protect\ref{fig:cka_alignment_additional} but without centering the kernel. Evolution of the uncentered kernel alignment between the tangent kernel and the class label kernel $K_Y = YY^T$ measured on a held-out test set for different architectures: \textbf{(left)} 6 layers of 80 hidden units MLP on MNIST \textbf{(middle)} VGG19 on CIFAR10 \textbf{(right)} Resnet18 on CIFAR10. We observe an increase of the alignment to the target function.}\label{fig:uncentered_alignment}
\end{figure*}

\subsection{Effect of depth on alignment}

In order to study the influence of the architecture on the alignment effect, we measure the CKA for different networks and different initialization as we increase the depth. The results in Fig \ref{fig:align_varying_depth} suggest that the alignment effect is magnified as depth increases. We also observe that the ratio of the maximum alignment between easy and difficult examples is increased with depth, but stays high for a smaller number of iterations.

\begin{figure*}
\centering
\includegraphics[width=\linewidth]{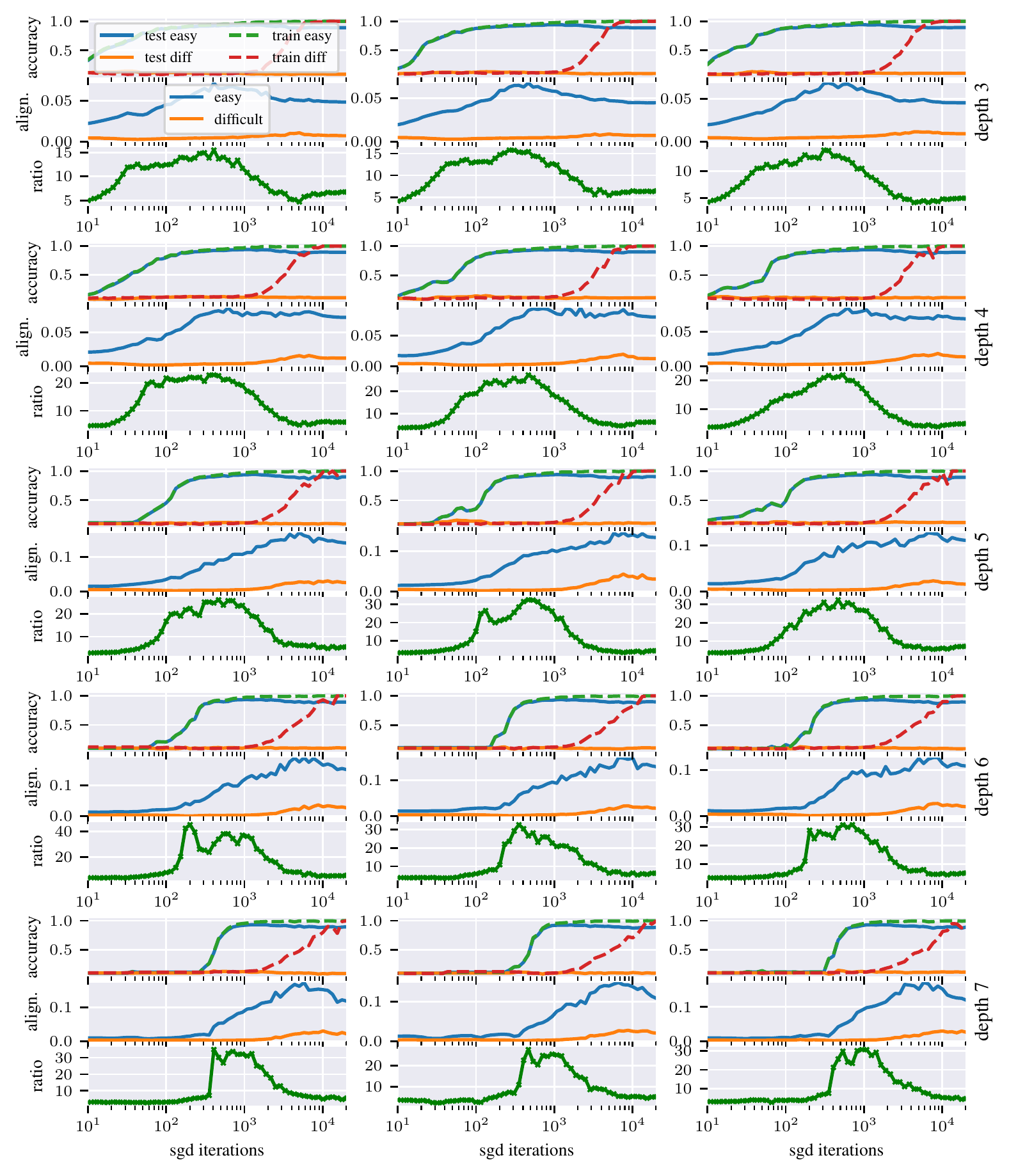}
\caption{\small Effect of depth on alignment. 10.000 MNIST examples with 1000 random labels MNIST examples trained with learning rate=0.01, momentum=0.9 and batch size=100 for MLP with hidden layers size 60 and  \textbf{(in rows)} varying depths \textbf{(in columns)} varying random initialization/minibatch sampling. As we increase the depth, the alignment starts increasing later in training and increases faster; and the ratio between easy and difficult alignments reaches a higher value.}\label{fig:align_varying_depth}
\end{figure*}

\subsection{Spectrum Plots with lower learning rate : Fig. \ref{appfig:funcclassTKspectrumwithtraceratios}}
\label{appendix:NTKSpectra}

\begin{figure*}[t]
\includegraphics[width=.33\linewidth]{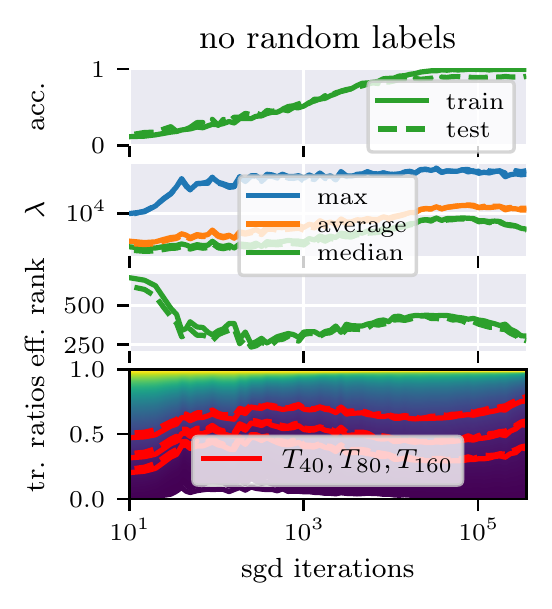}
\includegraphics[width=.33\linewidth]{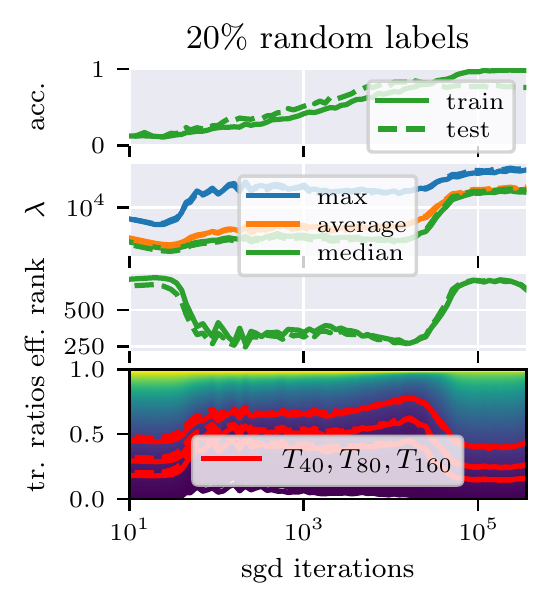}
\includegraphics[width=.33\linewidth]{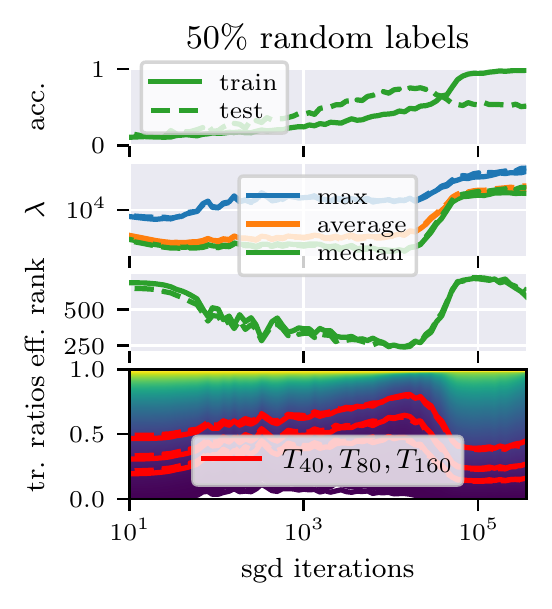}
\caption{\small Evolution of tangent kernel spectrum, effective rank and trace ratios  of a VGG19 trained by SGD with batch size $100$, learning rate $0.003$ and momentum $0.9$ on dataset \textbf{(left)} CIFAR10 and \textbf{(right)} CIFAR10 with 50\% random labels. We highlight the top 40, 80 and 160 trace ratios in \textbf{red}.}\label{appfig:funcclassTKspectrumwithtraceratios}
\end{figure*}

\end{document}